\documentclass[letterpaper]{article} 
\usepackage[preprint]{aaai2027} 

\usepackage[hyphens]{url} 
\usepackage{graphicx} 
\usepackage{natbib} 
\usepackage{caption} 
\usepackage{amsmath}
\usepackage{amssymb}
\usepackage{amsthm}

\usepackage{booktabs}
\usepackage{multirow}

\usepackage{xcolor}
\usepackage{enumitem}
\usepackage{pifont}

\usepackage{algorithm}
\usepackage{algorithmic}

\usepackage{newfloat}
\usepackage{listings}

\DeclareCaptionStyle{ruled}{
  labelfont=normalfont,
  labelsep=colon,
  strut=off
} 

\floatstyle{ruled}
\newfloat{listing}{tb}{lst}{}
\floatname{listing}{Listing}

\theoremstyle{plain}
\newtheorem{proposition}{Proposition}[section]

\theoremstyle{definition}

\newcommand{\cmark}{\textcolor{green!60!black}{$\checkmark$}}
\newcommand{\xmark}{\textcolor{red}{$\times$}}
\newcommand{\Epub}{\mathcal{E}^{\mathrm{pub}}}
\newcommand{\Epriv}[1]{\mathcal{E}_{#1}^{\mathrm{priv}}}
\newcommand{\Eview}[1]{\mathcal{E}_{#1}}

\usepackage{tcolorbox}
\tcbuselibrary{skins,breakable}
\usepackage{tabularx}

\definecolor{agA}{HTML}{1F4E79}
\definecolor{agB}{HTML}{B7472A}
\definecolor{agC}{HTML}{2E7D32}
\definecolor{okcol}{HTML}{2E7D32}
\definecolor{badcol}{HTML}{B71C1C}

\newcommand{\agA}[1]{\textcolor{agA}{\textbf{#1}}}
\newcommand{\agB}[1]{\textcolor{agB}{\textbf{#1}}}
\newcommand{\agC}[1]{\textcolor{agC}{\textbf{#1}}}
\newcolumntype{R}{>{\raggedright\arraybackslash}X}

\newtcolorbox{successcase}[1]{%
  enhanced, breakable,
  colback=okcol!3,
  colframe=okcol!50!black,
  arc=2pt, boxrule=0.5pt,
  title={\raggedright $\checkmark$\,\,#1},
  fonttitle=\bfseries\small,
  colbacktitle=okcol!55!black,
  coltitle=white,
  left=6pt, right=6pt, top=5pt, bottom=5pt,
  before skip=8pt, after skip=8pt,
}

\newtcolorbox{failurecase}[1]{%
  enhanced, breakable,
  colback=badcol!3,
  colframe=badcol!50!black,
  arc=2pt, boxrule=0.5pt,
  title={\raggedright $\times$\,\,#1},
  fonttitle=\bfseries\small,
  colbacktitle=badcol!55!black,
  coltitle=white,
  left=6pt, right=6pt, top=5pt, bottom=5pt,
  before skip=8pt, after skip=8pt,
}

\title{Diverse Evidence, Better Forecasts: \\ Multi-Agent Deliberation Under Information Asymmetry}

\author{Yuante Li\textsuperscript{\rm 1}\thanks{Equal Contribution.}\thanks{Corresponding Author.},\hspace{4pt} Yicheng Tao\textsuperscript{\rm 1}\footnotemark[1],\hspace{4pt} Kate Zhang\textsuperscript{\rm 1}\footnotemark[1], \\
Taozhi Wang\textsuperscript{\rm 1}, \hspace{4pt}Gefei Gu\textsuperscript{\rm 1}, \hspace{4pt}Yaxin Zhou\textsuperscript{\rm 2}
}

\affiliations{
\textsuperscript{\rm 1}School of Computer Science, Carnegie Mellon University \\
\textsuperscript{\rm 2}College of Engineering, Carnegie Mellon University \\
\texttt{yuantel@cs.cmu.edu}, \texttt{yichengtao@cmu.edu}, \texttt{katezhan@cs.cmu.edu}
}

\begin{document}
\maketitle
\begin{abstract}
Multi-agent systems are increasingly used for forecasting future events, as deliberation among multiple LLMs is believed to improve reasoning and calibration. Yet existing approaches overlook a critical design choice: what information each agent receives.
When all agents are given identical evidence, deliberation collapses into herding rather than genuine belief revision, leaving multi-agent systems little better than a single agent.
We identify this as a fundamental gap and propose \textit{designed information asymmetry} to close it: by partitioning evidence into a shared public pool and disjoint private subsets, we give each agent exclusive knowledge that can reach others only through deliberation.
We theoretically show that public/private evidence partitioning reduces inter-agent error correlation, and instantiate this principle in \textbf{InfoDelphi}, a framework combining relevance-aware evidence routing, rationale-based iterative deliberation, and calibrated asymmetric-shrinkage aggregation.
On \textsc{PolyGym}, a benchmark of 250 binary questions from real-world prediction markets, and on FutureX, InfoDelphi beats the strongest baseline by 13\% and 18\% in Brier score.
Ablation experiments confirm that removing information asymmetry eliminates most of the deliberation gains, so input diversity, not deliberation alone, drives the improvement.
\end{abstract}

\section{Introduction}
\label{sec:intro}

Prediction markets such as Polymarket\footnote{\url{https://polymarket.com/}} aggregate participants' beliefs into continuously updated probability estimates for future events. With every question resolved against a definitive ground truth, they provide a natural testbed for machine forecasting: \textit{\textbf{can an automated system reason over evidence and match human-level forecasting?}}

Recent work has shown that retrieval-augmented large language models (LLMs) can match or outperform expert forecasters on binary forecasting tasks, while outcome-driven fine-tuning can further improve the forecasting ability of LLMs themselves~\citep{halawi2024approaching,turtel2025llms}. But these advances rely primarily on single-agent, single-pass pipelines that treat the entire evidence corpus as a monolithic context window.
A natural extension is the \textit{multi-agent} setting, where a panel of LLMs deliberates over the evidence before producing a collective forecast, an approach that has shown promise across a range of reasoning and generation benchmarks~\citep{chen2024reconcile,wangmixture}.
Such frameworks draw inspiration from well-studied human mechanisms for eliciting and aggregating group judgments. In particular, the \textbf{Delphi} method uses structured, multi-round elicitation and controlled feedback to refine group judgments~\citep{DALKEY1969408}.

\begin{figure}[ht]
    \centering
    \includegraphics[width=1\linewidth]{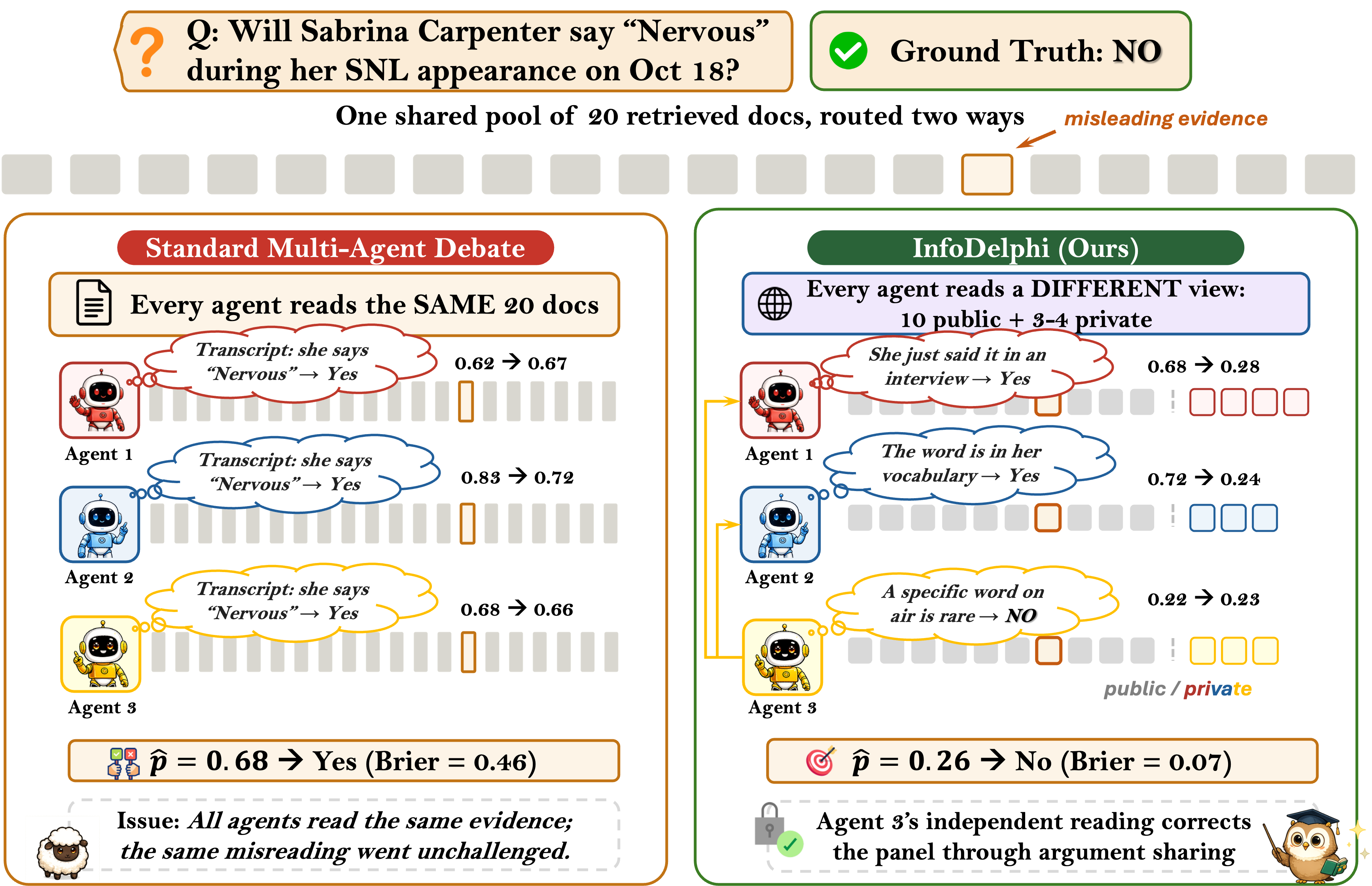}
    \caption{Effect of information asymmetry on a real \textsc{PolyGym} question. \textbf{Left}: agents reading the same evidence anchor on the same misleading document and converge on the wrong forecast (Brier 0.46). \textbf{Right}: under partitioned evidence one agent keeps an independent reading, and its cited argument corrects the panel (Brier 0.07). Appendix~\ref{sec:appendix_case} walks through this question in full.}
    \label{fig:intro}
\end{figure}

Effective collective reasoning, however, requires that participants hold \textit{informationally diverse} views.
 Classical jury-theorem results highlight competence and independence, while human experiments show that social influence can reduce opinion diversity and weaken the wisdom-of-crowds effect~\citep{Lorenz2011HowSI}.
LLM systems also exhibit substantial error dependence: when two models are both wrong, they often select the same incorrect answer~\citep{kim2025correlated}. Moreover, under a stochastic belief-update model, homogeneous-input debate forms a \textit{martingale} whose expected correctness does not improve across rounds~\citep{choi2025debate}.
Figure~\ref{fig:intro} illustrates this failure mode on a real prediction market question: when all agents receive identical evidence, they converge to the same incorrect forecast (left), whereas agents equipped with complementary private evidence can self-correct through deliberation (right).

To address this, we propose \textit{designed information asymmetry}, treating information asymmetry not as a limitation to overcome but as a principle to exploit.
We introduce \textbf{InfoDelphi}, a multi-agent forecasting framework that partitions evidence into a shared public pool and disjoint private subsets, giving agents distinct evidence views while preserving common ground for communication.
This mirrors expert panels in which participants contribute complementary knowledge rather than relying on identical briefing materials.
InfoDelphi places the most relevant documents in the public pool using BM25 and distributes the remaining evidence across agents.
During deliberation, agents exchange evidence-cited rationales rather than probabilities or labels, allowing cited private evidence to propagate across the panel.
Final forecasts are averaged and passed through an asymmetric shrinkage transform toward a NO-leaning prior, correcting systematic overconfidence on the affirmative side.
Together, these design choices enable deliberation to surface complementary information rather than merely reinforce shared beliefs.

Furthermore, existing forecasting benchmarks entangle retrieval quality with reasoning ability, making it impossible to assess whether gains come from better search or better deliberation.
We introduce \textsc{PolyGym}, a controlled benchmark of 250 binary Polymarket questions with fixed, pre-retrieved evidence shared across all methods.
On \textsc{PolyGym} and on 231 questions from the FutureX~\citep{zeng2025futurex} benchmark, InfoDelphi outperforms all single-agent and multi-agent baselines, improving Brier score over the strongest baseline by 13.2\% and 17.6\%, with every difference significant under a paired bootstrap.
Ablations and sensitivity analyses show that InfoDelphi is robust across design choices, but degrades significantly when private evidence is removed. Cross-model and deliberation-dynamics results further confirm that information asymmetry prevents herding and preserves useful forecast diversity.
These results support designed information asymmetry as a general design principle for multi-agent reasoning under uncertainty.

Our main contributions are as follows:
\begin{itemize}[nosep,leftmargin=20pt,topsep=2pt,itemsep=1pt,parsep=0pt]
  \item We propose \textit{designed information asymmetry} for multi-agent forecasting, and theoretically show that public/private evidence partitioning reduces inter-agent error correlation.
  \item We introduce \textbf{InfoDelphi}, which instantiates this principle through relevance-aware evidence routing, rationale-based deliberation, and a calibrated shrinkage aggregation layer whose parameters transfer across datasets.
  \item We release \textsc{PolyGym}, a controlled benchmark of 250 resolved Polymarket questions with frozen, leakage-audited evidence, which isolates information utilization from retrieval quality and can be repartitioned across agents at will.
  \item On \textsc{PolyGym}-250 and FutureX-231, InfoDelphi outperforms nine single- and multi-agent baselines with paired-bootstrap significance. Ablations identify information asymmetry as the primary source of the gains, and deliberation-dynamics analysis provides direct evidence of herding under homogeneous input.
\end{itemize}

\section{Related Work}
\label{sec:related}

\paragraph{LLM-based Forecasting.}
Early work established a substantial gap between LLM and human forecasting accuracy.
\citet{zou2022forecasting} find that language models achieve only 65\% accuracy on Autocast compared to 92\% for human aggregates, while \citet{schoenegger2023large} show that GPT-4 without augmentation does not significantly outperform a random baseline on Metaculus.
Retrieval-augmented generation~\citep{lewis2020retrieval} substantially closes this gap: \citet{yanautocast++} improve Autocast accuracy through zero-shot re-ranking and summarization of retrieved news, and \citet{halawi2024approaching} present a full pipeline achieving near-human performance on Polymarket and Metaculus.
More recent work focuses on training: \citet{turtel2025llms} fine-tune on 12,100 resolved Polymarket questions, and \citet{chandakscaling} train an 8B model via reinforcement
learning that matches proprietary LLMs on prediction benchmarks.
On the aggregation side, \citet{schoenegger2024wisdom} show that ensembling diverse
LLMs produces forecasts statistically indistinguishable from human forecasters.
Benchmark design has matured in parallel: ForecastBench~\citep{karger2024forecastbench}
and FutureX~\citep{zeng2025futurex} provide continuously updated evaluations with strict
temporal controls, and \citet{paleka2025pitfalls} identify systematic pitfalls including
temporal leakage and piggybacking on existing forecasts.
Despite this progress, existing pipelines predominantly rely on single-agent or homogeneous-input setups and do not study how evidence should be distributed across a deliberating panel.

\paragraph{Multi-Agent Reasoning and Deliberation.}
Multi-agent debate has emerged as a general strategy for improving LLM reasoning~\citep{du2023improving}.
\citet{wangself} show that sampling diverse chain-of-thought paths and selecting by majority vote improves accuracy, while \citet{wangmixture} propose Mixture-of-Agents, a layered architecture where later agents condition on all prior outputs.
\citet{chen2024reconcile} demonstrate that a round-table conference of diverse LLMs with confidence-weighted voting outperforms single-agent baselines across reasoning benchmarks.
However, recent theory suggests multi-agent debate is not uniformly beneficial.
\citet{choi2025debate} prove that under a Dirichlet-Categorical belief model, standard multi-agent debate forms a \textit{martingale}: expected correctness does not improve over rounds when agents receive identical inputs.
\citet{shin2026reasoning} formalize this via the Data Processing Inequality: closed-system deliberation forms a Markov chain, so mutual information with the ground truth can only decrease.
\citet{kim2025correlated} document that LLM errors are over 60\% correlated, implying that naive ensembling hits a non-zero error floor.
\citet{liu2024autonomous} study agents whose information is distributed across users and treat the resulting asymmetry as an obstacle to overcome through guided information exchange, whereas we create asymmetry by design to decorrelate forecasts.
Our work addresses these limitations by introducing designed information asymmetry to break correlation and convert the martingale into productive belief revision.

\paragraph{Information Aggregation and Collective Intelligence.}
The conditions under which groups outperform individuals are well studied.
\citet{surowiecki2004wisdom} highlights diversity, independence, decentralization, and effective aggregation as conditions supporting collective intelligence, while \citet{hong2004groups} show, under a formal problem-solving model, that functionally diverse teams can outperform teams of the individually best performers.
The DeGroot model~\citep{612bb50a-4bdd-3a32-b6eb-7837600cc9c4} formalizes iterative opinion updating and consensus formation, while the Delphi method~\citep{DALKEY1969408} uses structured, multi-round elicitation and feedback to refine group judgments.
For probabilistic forecasting, classical methods include the linear opinion pool~\citep{9ca5e582-1e4c-323f-9ab1-c62032d22e48} and log-odds aggregation~\citep{SATOPAA2014344}.
Meanwhile, social influence can reduce opinion diversity without improving group accuracy~\citep{Lorenz2011HowSI}, and informational cascades can lead agents to disregard private signals in favor of observed actions~\citep{bikhchandani1992theory}.
Our framework builds on these ideas by creating differentiated evidence views and enabling evidence-grounded rationale sharing before aggregation.

\section{Methodology}
\label{sec:method}

\subsection{Problem Formulation}
Let $q$ denote a binary forecasting question with resolution date $\tau$ and ground-truth outcome $y \in \{0,1\}$, where $y = 1$ means the question resolves YES. Each question comes with a retrieved evidence corpus $\mathcal{E} = \{e_1, \dots, e_n\}$, where each document $e_i$ consists of a title, a URL, and a text snippet, all retrieved before $\tau$. Given $q$ and $\mathcal{E}$, the task is to produce a probability forecast $\hat{p} \in [0,1]$ estimating $\Pr(y = 1 \mid q, \mathcal{E})$.

\subsection{Theoretical Motivation}
\label{sec:theory}

We begin by formalizing why information asymmetry improves collective forecasting, drawing on the bias-variance-diversity decomposition of the ensemble error.

\paragraph{Homogeneous input induces correlated errors.}
Consider $J$ agents each receiving the same evidence $\mathcal{E}$. Let $p_j$ denote agent $j$'s forecast and $y$ the ground truth. The ensemble forecast $\bar{p} = \frac{1}{J}\sum_j p_j$ has mean squared error
  \begin{equation}
    \mathrm{MSE}(\bar{p}) \;=\; \mathrm{Bias}^2 + \bar{\sigma}^2 - \mathrm{Div},
    \label{eq:bvd}
  \end{equation}
where $\mathrm{Bias}$ is the bias shared by the agents, $\bar{\sigma}^2 = \frac{1}{J}\sum_j \mathrm{Var}(p_j)$ is the average agent variance, and $\mathrm{Div} = \frac{1}{J}\sum_j \mathrm{Var}(p_j - \bar{p})$ is the diversity term~\citep{krogh1995neural,wood2023unified}.
The identity is exact for exchangeable agents, which ours are by construction: they share one model and one prompt, and differ only in which slice of the corpus they are dealt.
Appendix~\ref{sec:appendix_bvd} gives the derivation.
When all agents observe identical input, their prediction errors $\epsilon_j = p_j - y$ are driven by the same evidence and thus highly correlated: $\mathrm{Cov}(\epsilon_i, \epsilon_j) \approx \mathrm{Var}(\epsilon_j)$ for $i \neq j$.
In this regime the diversity term collapses toward zero, and the ensemble provides no benefit over a single agent. Recent empirical work confirms this: LLM errors are correlated across providers, and optimal aggregation under full correlation has a non-zero error floor~\citep{kim2025correlated,turkmen2026don}.

\paragraph{Information asymmetry reduces error correlation.}
We now quantify, under an idealized error model, how a public/private evidence partition reduces this correlation.

\begin{proposition}[Diversity-Induced Decorrelation]
    \label{prop:decorrelation}
    Let each agent's prediction error decompose as
    $\epsilon_j = \rho \cdot \epsilon^{\mathrm{pub}} + (1-\rho) \cdot \epsilon_j^{\mathrm{priv}}$,
    where $\epsilon^{\mathrm{pub}}$ is the error component driven by shared public evidence,
    $\epsilon_j^{\mathrm{priv}}$ is the component driven by agent $j$'s private evidence,
    and $\rho \in [0,1]$ is the public ratio.
    Let $\sigma_{\mathrm{pub}}^2 = \mathrm{Var}(\epsilon^{\mathrm{pub}})$ and assume:
    \begin{enumerate}[label=(\roman*),nosep,leftmargin=20pt,topsep=2pt,itemsep=1pt]
    \item $\mathrm{Cov}(\epsilon_i^{\mathrm{priv}}, \epsilon_j^{\mathrm{priv}}) = 0$ for $i \neq j$ \textit{(private--private independence)};
    \item $\mathrm{Cov}(\epsilon^{\mathrm{pub}}, \epsilon_k^{\mathrm{priv}}) = 0$ for all $k$ \textit{(public--private independence)};
    \item $\mathrm{Var}(\epsilon_j^{\mathrm{priv}}) = \sigma_{\mathrm{priv}}^2$ for all $j$ \textit{(symmetric private noise)}.
  \end{enumerate}
    Then:
    \begin{equation}
      \mathrm{Corr}(\epsilon_i, \epsilon_j) \;=\;
      \frac{\rho^2\,\sigma_{\mathrm{pub}}^2}
           {\rho^2\,\sigma_{\mathrm{pub}}^2 + (1{-}\rho)^2\,\sigma_{\mathrm{priv}}^2}.
      \label{eq:corr}
    \end{equation}
  \end{proposition}

\begin{proof}[Proof sketch]
By bilinearity of covariance and assumptions (i)--(ii), all cross-terms vanish, giving $\mathrm{Cov}(\epsilon_i, \epsilon_j) = \rho^2\sigma_{\mathrm{pub}}^2$.
By assumptions (ii)--(iii), the cross-term in the variance expansion also vanishes, so $\mathrm{Var}(\epsilon_j) = \rho^2\sigma_{\mathrm{pub}}^2 + (1{-}\rho)^2\sigma_{\mathrm{priv}}^2$, which equals $\mathrm{Var}(\epsilon_i)$ by symmetry.
Dividing covariance by $\sqrt{\mathrm{Var}(\epsilon_i)\mathrm{Var}(\epsilon_j)} = \mathrm{Var}(\epsilon_j)$ gives~\eqref{eq:corr}, and the full expansion is in Appendix~\ref{sec:appendix_proof}.
\end{proof}

\noindent
Proposition~\ref{prop:decorrelation} makes two points.
First, error correlation falls monotonically as $\rho$ decreases: with equal variance components ($\sigma_{\mathrm{pub}} =
\sigma_{\mathrm{priv}}$), Eq.~\eqref{eq:corr} equals $1$ at $\rho = 1$, $0.5$ at $\rho = 0.5$, and $0$ at $\rho = 0$.
Equal variances are assumed only for this illustration.
Eq.~\eqref{eq:corr} holds for any $\sigma_{\mathrm{pub}}, \sigma_{\mathrm{priv}} > 0$, and the correlation stays below one for every $\rho < 1$, so the variance ratio sets how fast decorrelation proceeds, not whether it occurs.
Second, when agents share everything ($\rho = 1$), their errors are identical and averaging the forecasts cancels nothing, so the ensemble adds nothing over a single agent, whereas any split leaves some independent error for averaging to cancel. The theory does not, however, say where the best split lies.

Two extensions in Appendix~\ref{sec:appendix_proof} show that the result does not hinge on this idealization.
The decomposition arises exactly when each document contributes an independent error and an agent averages over the documents it reads, in which case partitioning lowers the ensemble error variance at any fixed per-agent reading budget.
The independence assumption (i) can also be relaxed.
If private errors share a pairwise correlation $\rho_\epsilon < 1$, as agents built on one backbone plausibly do, the variance reduction shrinks with $\rho_\epsilon$ but remains strictly positive for every $\rho < 1$.

The other endpoint is degenerate too: at $\rho = 0$ agents share no common context through which to interpret one another's arguments. Empirically the interior is flat: Brier is statistically indistinguishable across $\rho \in \{0.3, 0.5, 0.7\}$, while $\rho = 1$ degrades sharply (Section~\ref{sec:ablation}).
This flatness is consistent with the design: deliberation with evidence pooling spreads cited private documents to the whole panel, so the final information sets depend only weakly on the initial split.
What changes at $\rho = 1$ is not the amount of private evidence but its existence: no document ever receives an independent reading, so deliberation has nothing to transport and can only reinforce the shared view.
We therefore treat $\rho$ as a robustness parameter rather than a tuned hyperparameter, and fix $\rho = 0.5$ a priori, which splits the corpus evenly between the public pool and the private subsets taken together.

\paragraph{Rationale sharing enables cross-agent information transfer.}
Even with disjoint evidence, agents can benefit from each other's private signals only through the between-round channel. By the Data Processing Inequality~\citep{shin2026reasoning}, given the question $q$ and the model parameters $\theta$, an agent's revised forecast carries no more information about the ground truth than its own view plus what that channel delivers: $I(y;\, p_j^{(r+1)} \mid q, \theta) \leq I(y;\, \Eview{j},\, h^{(r)} \mid q, \theta)$.
What the channel carries therefore sets the ceiling on what deliberation can add. A numeric estimate compresses agent $k$'s entire view into one scalar: it cannot tell the panel which documents drive a disagreement, and in practice it acts as a social anchor rather than as evidence. A rationale with citations transmits the claims themselves, and evidence pooling then delivers the cited documents in full, so private evidence reaches the rest of the panel with little loss. We design the channel to carry arguments and the documents behind them, and to withhold every number that debate and classical Delphi feed back, namely probabilities, labels, and the panel mean or median.

 \begin{figure*}[t]
      \centering
      \includegraphics[width=\textwidth]{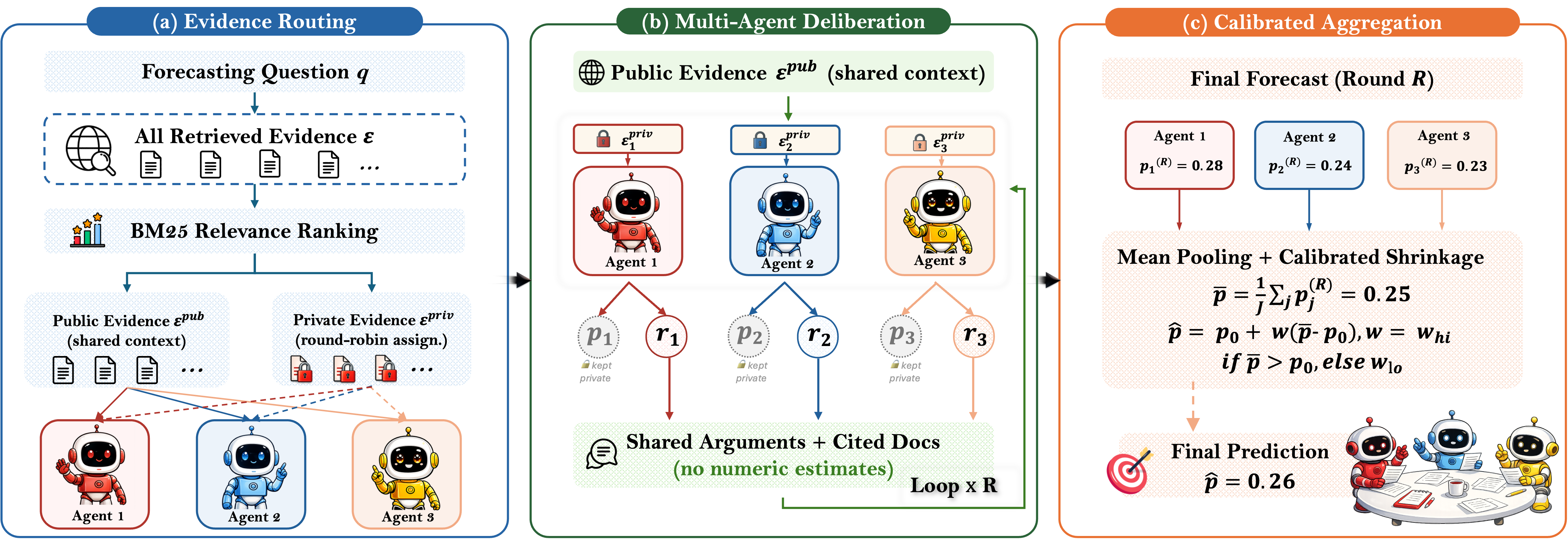}
    \caption{Overview of the \textbf{InfoDelphi} framework.
    (a)~\textbf{Evidence Routing}: documents are ranked by BM25 relevance and
    distributed round-robin into a public pool $\Epub$ shared by all agents and disjoint private subsets $\Epriv{j}$ held by one agent each.
    (b)~\textbf{Multi-Agent Deliberation}: in each of $R$ rounds, agent $j$ reads its own view and produces a probability estimate $p_j$ and an evidence-cited argument $r_j$. Only $r_j$ and the documents it cites are broadcast, while $p_j$ stays private, so agents exchange arguments, never numeric estimates.
    (c)~\textbf{Calibrated Aggregation}: the final-round forecasts $p_j^{(R)}$ are averaged into $\bar{p}$ and passed through an asymmetric shrinkage toward a NO-leaning prior $p_0$, yielding the final forecast $\hat{p}$ and correcting systematic overconfidence. Numbers in (c) trace the worked example of
    Figure~\ref{fig:intro}.}
      \label{fig:arch}
  \end{figure*}

\subsection{The InfoDelphi Framework}

\paragraph{Overview.}
Figure~\ref{fig:arch} shows the InfoDelphi pipeline.
\textit{Evidence routing} partitions $\mathcal{E}$ into a shared public pool and disjoint private subsets: the private subsets create the decorrelated signals of Proposition~\ref{prop:decorrelation}, while the public pool preserves the common ground agents need to interpret one another.
\textit{Iterative deliberation} has agents exchange evidence-cited arguments across rounds, the channel through which private signals reach the rest of the panel.
\textit{Calibrated aggregation} pools the final-round forecasts by the mean and applies an asymmetric shrinkage toward a NO-leaning prior, correcting a systematic overconfidence that the information structure alone does not remove.

\paragraph{Evidence Routing.}
\label{sec:routing}
We consider a panel of $J$ agents $\{A_1, \dots, A_J\}$ and a public ratio $\rho \in [0,1]$. The corpus is partitioned into a shared public subset $\Epub$ and $J$ disjoint private subsets:
  \begin{equation}
    \mathcal{E} \;=\; \Epub \;\sqcup\; \bigsqcup_{j=1}^{J} \Epriv{j},
    \qquad |\Epub| \;=\; \lfloor \rho \cdot |\mathcal{E}| \rfloor.
    \label{eq:partition}
  \end{equation}
Agent $A_j$ observes the view $\Eview{j} = \Epub \cup \Epriv{j}$. Diversity comes from the disjointness of the private subsets, and BM25 relevance scores $s_i = \mathrm{BM25}(e_i, q)$ decide which documents go where.
The top-$|\Epub|$ scoring documents form $\Epub$: the documents most relevant to the question are the ones every agent needs in order to interpret the others' arguments, so they become the common ground. The remaining documents are sorted by $s_i$ and dealt to the agents in round-robin order, so that each private subset spans the full range of
relevance instead of one agent holding all the strong documents.

\paragraph{Iterative Deliberation with Rationale Sharing.}
  \label{sec:deliberation}
  At each round $r \in \{1, \dots, R\}$, agent $A_j$ produces a structured forecast
  \begin{equation}
      f_j^{(r)} \;=\; \bigl(p_j^{(r)},\, \ell_j^{(r)},\, r_j^{(r)},\, \mathcal{C}_j^{(r)}\bigr),
  \end{equation}
where $p_j^{(r)} \in [0,1]$ is a probability estimate, $\ell_j^{(r)} \in \{0,1\}$ a categorical prediction, $r_j^{(r)}$ a free-text rationale, and $\mathcal{C}_j^{(r)} \subseteq \Eview{j} \cup \bigcup_{k \neq j} \mathcal{C}_k^{(r-1)}$ the documents that rationale cites, drawn from the agent's own view and from the documents pooled in the previous round (with $\mathcal{C}_k^{(0)} = \varnothing$). The agent conditions on its own view and on a summary of the previous round,
  \begin{equation}
      f_j^{(r)} \;\sim\; \mathrm{LLM}\!\left(q,\;\Eview{j},\;h^{(r-1)}\right),
      \label{eq:agent}
  \end{equation}
  with $h^{(0)} = \varnothing$ and, for $r \geq 2$,
  \begin{equation}
    h^{(r-1)} \;=\; \bigl\{\bigl(r_k^{(r-1)},\, \mathcal{C}_k^{(r-1)}\bigr)\bigr\}_{k \neq j}.
    \label{eq:share}
  \end{equation}
 Agents exchange arguments and the documents behind them, so convergence has to be earned through evidence rather than induced by a numeric anchor. Exchanging point estimates instead reduces this channel to one scalar per agent, and Section~\ref{sec:ablation} ablates that variant directly.
 Because $\mathcal{C}_k^{(r-1)}$ travels with the rationale, a document held privately by one agent becomes citable context for the whole panel once that agent relies on it (\emph{evidence pooling}), and the prompt correspondingly requires every claim to name the \texttt{doc\_id} it rests on. The system prompt has agents estimate a base rate and check resolution criteria before committing to a probability (the \emph{calibrated prompt} in Appendix~\ref{sec:appendix_prompts}).

\paragraph{Calibrated Aggregation.}
\label{sec:aggregation}
Given the final-round forecasts $\{p_j^{(R)}\}_{j=1}^{J}$, InfoDelphi first pools them by the mean, $\bar{p} = \frac{1}{J}\sum_j p_j^{(R)}$, and then applies an \emph{asymmetric shrinkage} toward a prior $p_0$:
  \begin{equation}
    \hat{p} \;=\; p_0 + w\,(\bar{p} - p_0), \qquad
    w = \begin{cases} w_{\mathrm{hi}} & \bar{p} > p_0, \\ w_{\mathrm{lo}} & \bar{p} \le p_0, \end{cases}
    \label{eq:shrink}
  \end{equation}
with $p_0 = 0.30$, $w_{\mathrm{lo}} = 0.8$, and $w_{\mathrm{hi}} = 0.5$.
The design responds to a systematic bias of LLM forecasters~\citep{tian2023just}: they are overconfident on the affirmative side of low-base-rate questions, while the rest of the range is close to calibrated, so forecasts above the prior are pulled harder ($w_{\mathrm{hi}} < w_{\mathrm{lo}}$) than forecasts below it.
The three values are set on the \textsc{PolyGym} questions themselves from their reliability profile, with $p_0$ near the empirical base rate and the stronger pull on the overshooting side, and are applied to FutureX unchanged.
The calibrated \textsc{PolyGym} results are therefore in-sample with respect to these three values, while the FutureX results are fully held out.
A 5-fold cross-fit that re-selects the setting among the candidates of Appendix Table~\ref{tab:shrink_sensitivity} on four folds and scores the fifth yields 0.1558, no worse than the fixed setting (0.1575), so the \textsc{PolyGym} result does not hinge on in-sample selection.
Section~\ref{sec:calibration_study} shows that this correction transfers across datasets where a standard recalibration fit does not.

\section{Experimental Setup}
\label{sec:setup}

\paragraph{Datasets.}
Current LLM forecasting systems~\citep{halawi2024approaching,turtel2025llms} primarily emphasize \textit{retrieval quality}, while an equally important capability is \textit{information utilization}: \textbf{how agents interpret, distribute, and collaboratively reason over retrieved evidence.}
To study this dimension in isolation, we construct \textbf{\textsc{PolyGym}}, a controlled forecasting benchmark of 250 resolved binary Polymarket questions (resolution dates 2025-09 to 2025-12) spanning Sports, Business, Politics, and other domains.
Each question is paired with up to 20 pre-retrieved documents collected before market resolution under a strict temporal cutoff.
The resulting evidence pool is frozen and shared across all
methods, but can be repartitioned across agents, allowing us to vary the
information structure without changing the underlying corpus.
To test generalization beyond Polymarket, we additionally evaluate on 231 resolved questions from the FutureX benchmark~\citep{zeng2025futurex}, constructed with the same fixed-evidence protocol. More details of the two datasets are provided in Appendix~\ref{sec:appendix_dataset}.

\paragraph{Baselines.}
We compare against two categories of baselines, all run with the same backbone model and the same evidence pools, with details provided in Appendix~\ref{sec:appendix_details}.
\ding{182} \textit{Single-agent methods}: Chain-of-Thought~\citep{wei2022chain}, Self-Consistency~\citep{wangself}, Sequential Bayesian~\citep{shi2023language}, Superforecaster~\citep{karger2024forecastbench}, and the structured scratchpad pipeline of Halawi et al.~\citep{halawi2024approaching}.
\ding{183} \textit{Multi-agent methods with homogeneous input}: Standard Debate~\citep{du2023improving}, MoA~\citep{wangmixture}, Crowd Ensemble~\citep{schoenegger2024wisdom}, and AIA Forecaster~\citep{alur2025aia}.

\paragraph{Metrics and significance.}
We evaluate forecasting performance along three complementary dimensions: \textit{calibration}, \textit{classification}, and \textit{ranking}. We report the \textbf{Brier score},
$\mathrm{BS} = \frac{1}{N}\sum_{k=1}^{N}(\hat{p}_k - y_k)^2$,
as the primary metric of probabilistic forecast quality, where lower values are better. We also report \textbf{accuracy}, defined as the fraction of questions for which $\mathbf{1}[\hat{p} \geq 0.5]$ matches $y$, to measure binary decision quality. Finally, we report \textbf{AUC}, the probability that a randomly selected YES question receives a higher forecast than a randomly selected NO question. Because AUC is invariant to monotone recalibration, including the calibration layer in Eq.~\eqref{eq:shrink}, it isolates ranking performance from calibration. For statistical significance, we focus on the Brier score and perform a paired bootstrap over questions for every pairwise comparison, using 5{,}000 resamples and a fixed random seed. We report the 95\% confidence interval of the mean Brier-score difference and consider a difference significant when the interval excludes zero.

\paragraph{Implementation.}
We use \texttt{gpt-5.4-mini}, accessed through the OpenRouter API\footnote{\url{https://openrouter.ai/}}, as the backbone model for all main experiments over the same fixed pre-resolution evidence pools.
Unless otherwise specified, InfoDelphi uses $J{=}3$ agents and $R{=}2$ deliberation rounds, with a public-evidence ratio of $\rho{=}0.5$, BM25-based relevance routing, argument sharing with cited documents, a sampling temperature of $0.7$, and a fixed random seed.
The calibration layer in Eq.~\eqref{eq:shrink} uses the values given in Section~\ref{sec:aggregation}.
Appendices~\ref{sec:appendix_prompts} and~\ref{sec:appendix_details} provide the prompts and baseline configurations.

\section{Results}
\label{sec:results}
\subsection{Main Results}

Table~\ref{tab:main} compares InfoDelphi against all nine baselines on both datasets.

\begin{table}[t]
\centering
\small
\setlength{\tabcolsep}{1.5pt}
\renewcommand{\arraystretch}{1.05}
\caption{Main results on \textsc{PolyGym}-250 and FutureX-231 (leading zeros omitted).
InfoDelphi (w/o calibration) is the deliberated mean before the calibration transform of Eq.~\eqref{eq:shrink}, whose monotonicity leaves AUC identical for our two rows. $\dagger$: InfoDelphi (full) is significantly better under the paired bootstrap. Best in \textbf{bold}, second best \underline{underlined}.}
\label{tab:main}
\begin{tabular}{@{}lcccccc@{}}
\toprule
& \multicolumn{3}{c}{\textsc{PolyGym}-250}
& \multicolumn{3}{c}{FutureX-231} \\
\cmidrule(lr){2-4}
\cmidrule(lr){5-7}
Method
& Brier & Acc. & AUC
& Brier & Acc. & AUC \\
\midrule
\multicolumn{7}{@{}l}{\textit{Single-agent}} \\
Chain-of-Thought~\citeyearpar{wei2022chain}$^\dagger$
& .2063 & .696 & .678 & .2611 & .606 & .663 \\
Self-Consistency~\citeyearpar{wangself}$^\dagger$
& .1965 & .712 & .699 & .2713 & .550 & .630 \\
Sequential Bayesian~\citeyearpar{shi2023language}$^\dagger$
& .2523 & .540 & .591 & .3043 & .485 & .534 \\
Superforecaster~\citeyearpar{karger2024forecastbench}$^\dagger$
& .1937 & .700 & .722 & .2596 & .636 & .672 \\
Halawi et al.~\citeyearpar{halawi2024approaching}$^\dagger$
& .1944 & .716 & .711 & .2654 & .567 & .629 \\
\midrule
\multicolumn{7}{@{}l}{\textit{Multi-agent, homogeneous input}} \\
Standard Debate~\citeyearpar{du2023improving}$^\dagger$
& .1815 & .720 & .711 & .2433 & .602 & \underline{.694} \\
MoA~\citeyearpar{wangmixture}$^\dagger$
& .1905 & .720 & .716 & .2944 & .545 & .599 \\
Crowd Ensemble~\citeyearpar{schoenegger2024wisdom}$^\dagger$
& .1992 & .700 & .700 & .2888 & .524 & .599 \\
AIA Forecaster~\citeyearpar{alur2025aia}$^\dagger$
& .1897 & .704 & \underline{.729} & .2764 & .550 & .635 \\
\midrule
\multicolumn{7}{@{}l}{\textit{Ours}} \\
InfoDelphi (w/o calibration)
& \underline{.1647} & \textbf{.748} & \textbf{.757}
& \underline{.2212} & \underline{.688} & \textbf{.711} \\
\textbf{InfoDelphi (full)}
& \textbf{.1575} & \underline{.744} & \textbf{.757}
& \textbf{.2006} & \textbf{.701} & \textbf{.711} \\
\bottomrule
\end{tabular}
\end{table}

\paragraph{\ding{224} Overall performance.}
InfoDelphi achieves the best score across every metric on both datasets in Table~\ref{tab:main}, and its Brier improvement over \emph{every} baseline is statistically significant on both benchmarks (Appendix Table~\ref{tab:significance}).
Against the strongest baseline the margin is 13.2\% on \textsc{PolyGym} and 17.6\% on FutureX.
The improvement stems from two complementary mechanisms: \textbf{evidence routing} gives each agent a partially private view of the corpus, creating the decorrelated signals that Proposition~\ref{prop:decorrelation} requires, and \textbf{rationale-based deliberation} propagates those private signals across the panel.
Both act before any calibration: the ensemble already leads all baselines \emph{without} the calibrated-aggregation layer, so the ranking among methods is established by the information structure alone.
Calibration then widens an advantage it did not create.
AUC makes the same point in calibration-free terms: because AUC is unaffected by any monotone recalibration, InfoDelphi's leading AUC on both benchmarks reflects a pure ranking advantage that no post-hoc recalibration of a baseline could recover.
A constant forecast at the empirical YES rate (23.2\% on \textsc{PolyGym} and 35.5\% on FutureX) scores a Brier of 0.1782 and 0.2290.
No baseline improves on this reference on either dataset, whereas InfoDelphi does so on both, already before calibration.

\paragraph{\ding{224} Information diversity, not agent count, drives multi-agent gains.}
The homogeneous-input baselines are comparable to InfoDelphi in inference budget and in several cases far larger: Crowd Ensemble uses five agents and AIA Forecaster eleven, against our three agents over two rounds.
Against Standard Debate, which makes the same six calls per question, InfoDelphi reads about 20\% fewer input tokens (17.7k versus 22.0k), because each agent reads only its own view instead of the full corpus (Appendix Table~\ref{tab:cost}). Nevertheless, on FutureX, all homogeneous-input baselines except Standard Debate underperform the strongest \emph{single-agent} baseline, with the same pattern appearing less strongly on \textsc{PolyGym}.
This result is consistent with our theoretical analysis. When agents receive identical evidence, additional agents and rounds repeatedly process the same information rather than introduce independent evidence. Inter-agent errors therefore remain highly correlated, limiting the benefit of aggregation and producing diminishing returns (Proposition~\ref{prop:decorrelation}, see also the martingale analysis of homogeneous debate in \citealp{choi2025debate}).
Standard Debate is a partial exception, indicating that iterative belief revision can still improve performance. However, because all agents operate on the same evidence, the resulting gains remain limited relative to those of InfoDelphi.

\subsection{Ablation and Sensitivity Analysis}
\label{sec:ablation}
We examine where the gains come from and how much they depend on our specific choices: we first ablate each core design component, then vary every remaining hyperparameter one at a time around the default configuration (details in Appendix~\ref{sec:appendix_details}).
Both analyses give the same answer: the advantage rests on the presence of private evidence, and nothing else requires tuning.

\begin{table}[t]
\centering
\small
\setlength{\tabcolsep}{4.0pt}
\caption{Component ablation on \textsc{PolyGym}-250. ``Brier'' is the mean of the final agent forecasts and ``$+$\,calib.'' applies the fixed calibration transform, so removing the calibration layer corresponds to reading the Brier column. ``$\rho{=}1$'' reproduces the Standard Debate configuration of our pipeline and is the only removal that is statistically significant ($^\dagger$).}
\label{tab:component_ablation}
\begin{tabular}{@{}lccc@{}}
\toprule
Configuration & Brier$\downarrow$ & $+$\,calib. & AUC$\uparrow$ \\
\midrule
InfoDelphi (full)
  & \textbf{0.1647} & \textbf{0.1575} & \textbf{0.757} \\
w/o deliberation ($R{=}1$)
  & 0.1698 & 0.1596 & 0.744 \\
w/o asymmetry ($\rho{=}1$)$^\dagger$
  & 0.1815 & 0.1650 & 0.711 \\
w/o argument sharing (numeric only)
  & 0.1687 & 0.1585 & 0.751 \\
w/o relevance routing (random split)
  & 0.1665 & 0.1586 & 0.746 \\
\bottomrule
\end{tabular}
\end{table}

\paragraph{\ding{224} Creating asymmetry is the largest factor.}
The largest degradation comes from removing information asymmetry: giving every agent the full pool with full sharing reproduces the Standard Debate configuration with prompt, protocol, and aggregation held fixed, and it is the only change whose paired CI excludes zero (Table~\ref{tab:component_ablation}).
Because prompt and aggregation are held fixed, the gap over the strongest baseline reflects what the agents read and exchange, and the same ordering holds on FutureX.
Keeping $\rho{=}1$ while withholding numeric probabilities (exchanging cited rationales only) lowers the Brier score of this configuration from 0.1815 to 0.1649 on \textsc{PolyGym}, while on FutureX the partitioned panel remains significantly better than this variant (0.2212 versus 0.2426), so the rationale channel and the partition both contribute.
Removing deliberation costs less: partitioning creates the \emph{potential} for improvement, and deliberation realizes it.
Downgrading argument exchange to numeric-only sharing or relevance routing to a random split increases the raw Brier score and lowers AUC, a ranking-quality loss that no monotone recalibration can recover.
That a random split costs so little is itself informative: what the partition has to supply is an independent first reading of \emph{some} document, not a particular allocation of the strong ones, which is why placing the top-ranked documents in the shared pool costs nothing.

\paragraph{\ding{224} No hyperparameter is load-bearing.}
Figure~\ref{fig:protocol_sensitivity} varies the deliberation depth, the public-evidence ratio, the sampling temperature, and the random seed around the default configuration.
No variation differs significantly from the default: every paired bootstrap CI against the default configuration straddles zero.
A third round yields a small directional gain, so deliberation saturates once private signals have propagated, and we keep $R{=}2$ for its lower cost.
The interior of the public-evidence ratio is flat, and temperature is statistically indistinguishable anywhere from $0$ to $1$.
The temperature panel further pins down where the load-bearing diversity comes from: even at $T{=}0$, where sampling noise vanishes entirely, agents retain 72\% of their usual round-one disagreement and performance is unchanged, so the diversity that matters is induced by the evidence partition rather than by decoding randomness.
Likewise in Appendix~\ref{sec:appendix_shrinkage}, every tested shrinkage setting improves over the untransformed mean on both datasets, although the best coefficients differ by dataset.
The only cliff in the entire design space is the homogeneous endpoint $\rho{=}1$, the same configuration that dominated the ablation: what matters is that agents hold \emph{some} private evidence, not how much or how it is sampled.

\begin{figure}[t]
\centering
\includegraphics[width=0.45\textwidth]{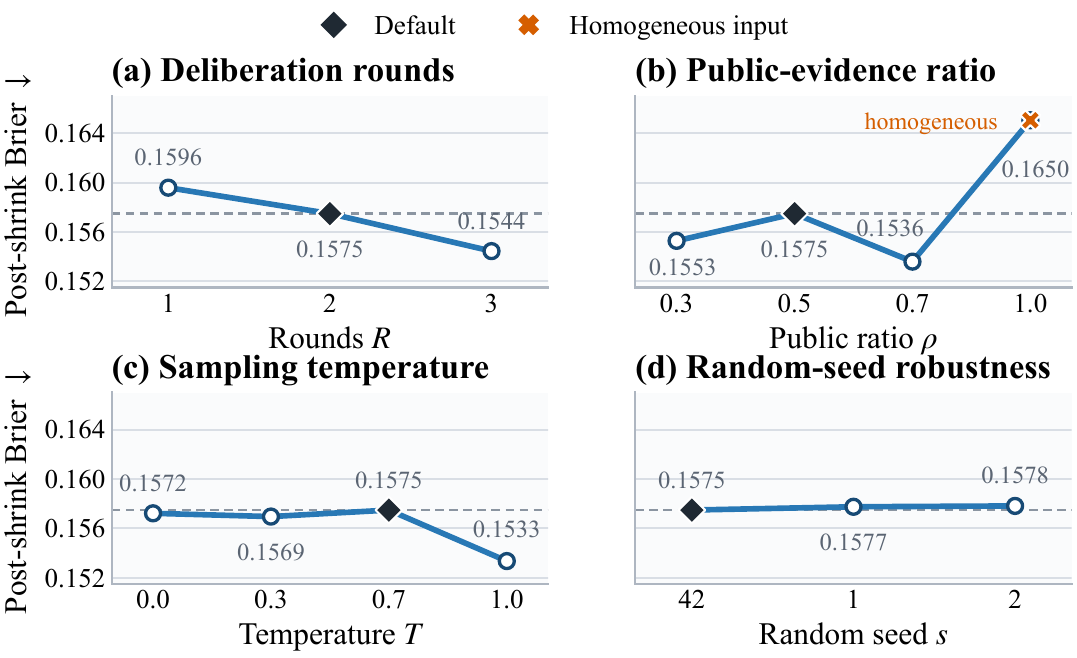}
\caption{Parameter sensitivity of InfoDelphi on \textsc{PolyGym}-250 after the fixed calibration transform of Eq.~\eqref{eq:shrink}. Diamonds denote the default configuration, and the cross marks the homogeneous-input condition. All panels share the same vertical scale, and lower is better.}
\label{fig:protocol_sensitivity}
\end{figure}

\subsection{Detailed Analysis}
\label{sec:dynamics}
\label{sec:calibration_study}

\paragraph{\ding{224} Homogeneous deliberation herds, while partitioned deliberation preserves diversity.}
Figure~\ref{fig:herding} traces why the strongest multi-agent baseline loses to InfoDelphi despite an identical protocol.
The two information structures start from comparable per-question disagreement, and a single round of deliberation drives them apart.
Under homogeneous input, disagreement collapses by roughly three quarters and the agents' errors become almost perfectly correlated: the panel effectively merges into one agent, and averaging can no longer cancel errors.
With partitioned evidence, agents converge far less and retain measurably more independent error, preserving the diversity term of Eq.~\eqref{eq:bvd}.
This is the mechanism Proposition~\ref{prop:decorrelation} points to, in the direction it predicts: private evidence gives each agent an independent basis for conviction that social pressure cannot override. The effect is directional rather than as large as the idealized model suggests, since agents sharing a backbone and a public pool remain correlated for reasons the partition cannot touch.
The preserved diversity is not timidity.
After the final round, InfoDelphi's agents are as extreme as Standard Debate's but disagree with each other far more.
Asymmetry decorrelates the direction of confidence rather than reducing it, so mean pooling averages complementary judgments instead of three copies of one belief.

\begin{figure}[t]
    \centering
    \includegraphics[width=\columnwidth]{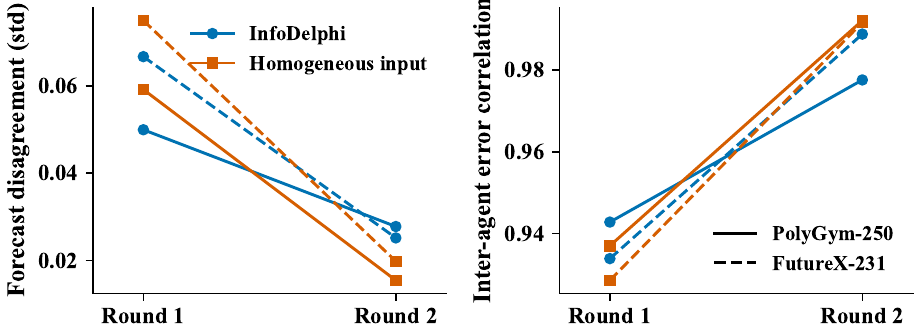}
    \caption{Deliberation dynamics with and without information asymmetry.
    Left: mean per-question standard deviation of agent forecasts. Right: mean
    pairwise inter-agent error correlation. Deliberation under homogeneous input
    collapses forecast diversity (herding), while partitioned evidence preserves
    independent signal after deliberation on both datasets.}
    \label{fig:herding}
\end{figure}

\paragraph{\ding{224} The calibration correction transfers because the bias is model-level, not dataset-level.}
The calibration parameters are set on \textsc{PolyGym} and applied to FutureX unchanged, where they deliver a comparable improvement, even though the best
FutureX-specific coefficients differ   (Appendix Table~\ref{tab:shrink_sensitivity}).
A correction that transfers across question distributions with different base rates captures a property of the forecaster rather than an artifact of one dataset.
The fitted Platt parameters, by comparison, differ across the two datasets, and even their in-sample fits stay behind the fixed asymmetric transform (Appendix Table~\ref{tab:platt}).
Appendix~\ref{sec:appendix_control} closes the loop by granting the identical transform to every baseline.
It helps them too, yet InfoDelphi remains best on both datasets, and the residual advantage is what calibration cannot supply: the decorrelated information structure.

\subsection{Cross-Model Generalization}
\label{sec:crossmodel}

\newcommand{\micon}[1]{\raisebox{-1pt}{\includegraphics[height=9pt]{#1}}}
\newcommand{\mname}[2]{\micon{#1}\hspace{2pt}{#2}}

\begin{table}[t]
\centering
\small
\setlength{\tabcolsep}{3.5pt}
\caption{Cross-model generalization of InfoDelphi.}
\label{tab:crossmodel}
\begin{tabular}{@{}lcccc@{}}
\toprule
& \multicolumn{2}{c}{\textsc{PolyGym}} &
  \multicolumn{2}{c}{FutureX} \\
\cmidrule(lr){2-3}\cmidrule(lr){4-5}
Backbone & Brier$\,\downarrow$ & Acc$\,\uparrow$ & Brier$\,\downarrow$ & Acc$\,\uparrow$ \\
\midrule
\mname{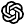}{GPT-5.4-mini} & \textbf{0.1575} & 0.744
& \textbf{0.2006} & \textbf{0.701} \\
\mname{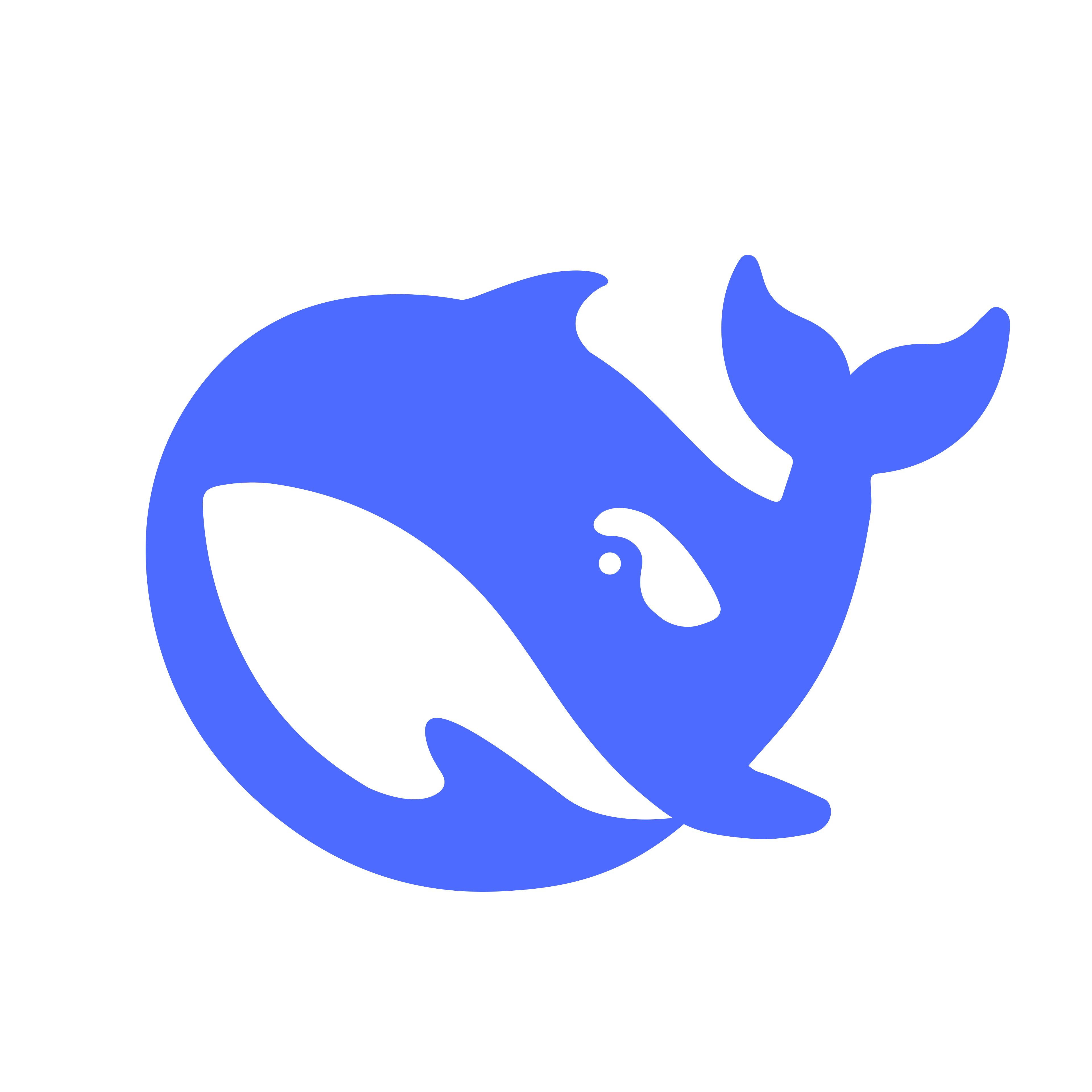}{DeepSeek-V3.2}
& 0.1626 & \textbf{0.760}
& 0.2029 & 0.662 \\
\mname{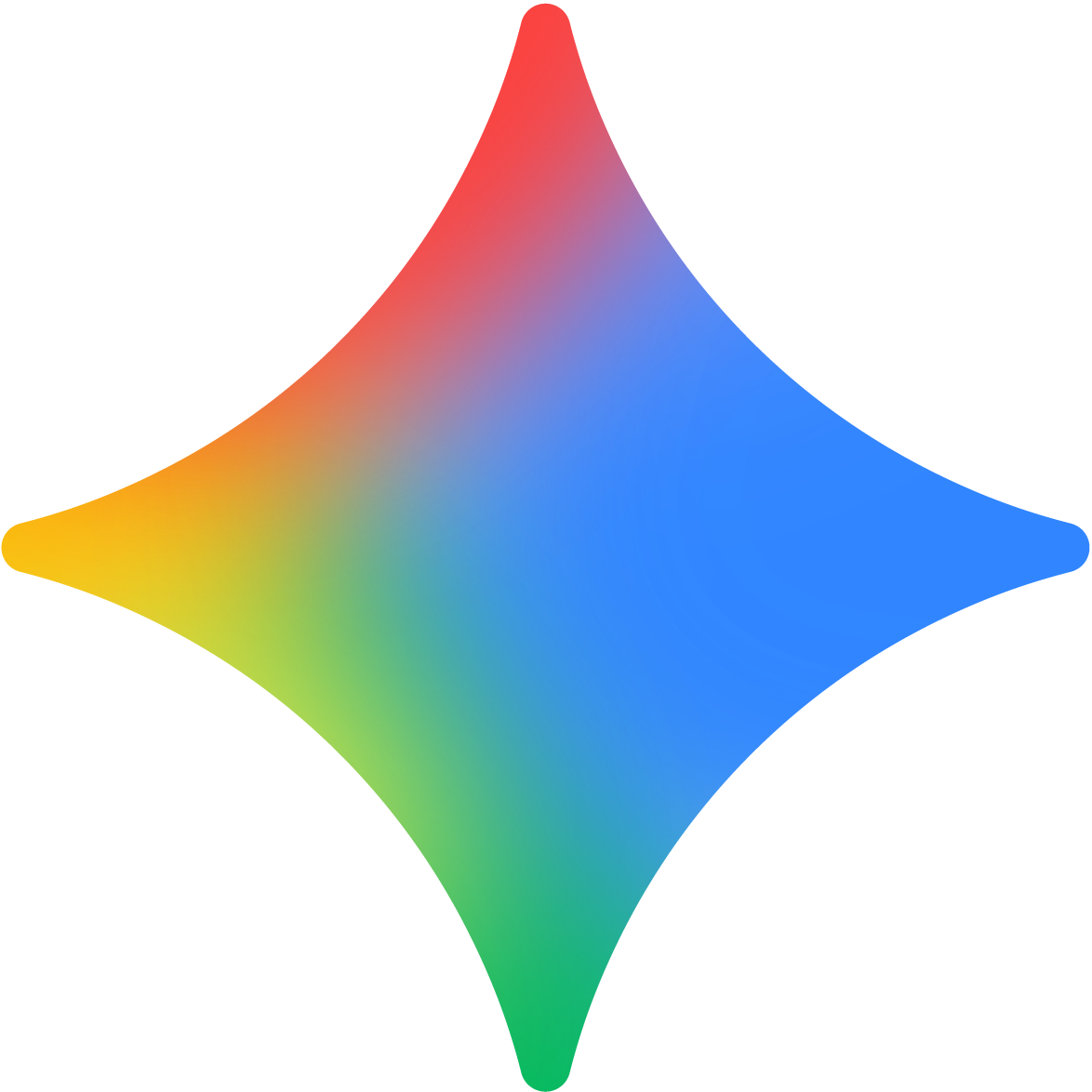}{Gemini-3.1-flash-lite}
& 0.1644 & 0.756
& 0.2040 & 0.688 \\
\bottomrule
\end{tabular}%
\end{table}

To assess whether InfoDelphi's gains depend on a particular backbone model, we evaluate the full framework with two additional backbones, DeepSeek-V3.2~\citep{liu2025deepseek} and Gemini-3.1-flash-lite~\citep{googledeepmind2026gemini31flashlite}, as reported in Table~\ref{tab:crossmodel}.
Across both datasets, InfoDelphi with each of the three backbones outperforms the strongest baseline of Table~\ref{tab:main}, while the performance differences among the backbones are modest relative to their shared margin over it.
These results indicate that InfoDelphi's advantage is robust to the choice of backbone, although the baselines were run with \texttt{gpt-5.4-mini} only, so the comparison does not by itself attribute each margin to the information structure.

\section{Conclusion}
\label{sec:conclusion}
We presented \textbf{InfoDelphi}, a multi-agent forecasting framework built on a simple but overlooked insight: collective reasoning only outperforms individual reasoning when participants hold different information.
By treating information asymmetry as a first-class design principle, InfoDelphi converts redundant deliberation into productive belief revision, beating nine single- and multi-agent baselines on two benchmarks and transferring across model families without re-tuning.
Probabilistic forecasting is a demanding test of this principle, requiring synthesis of noisy, contradictory evidence into calibrated estimates under irreducible uncertainty.
That designed information asymmetry yields robust improvements in this setting suggests the principle extends to other multi-agent systems that reason over shared evidence.

\section*{Acknowledgments}
Yuante Li gratefully acknowledges support from Carnegie Mellon University through the Graduate Student Assembly (GSA)/Provost Conference Funding Award and the Language Technologies Institute (LTI) Conference Travel Support.

\bibliography{custom}

\clearpage

\appendix
\section{Discussion}
\paragraph{Key insight.}
Our central finding is that \textit{designed information asymmetry} is necessary for effective multi-agent deliberation.
When all agents receive identical evidence, iterative discussion provides little measurable benefit, as agents merely reinforce correlated beliefs rather than contribute complementary knowledge.
This aligns with the martingale characterization of homogeneous-input debate and suggests that the widely adopted practice of giving all agents the same context limits the value of multi-agent interaction.

\paragraph{Future work.}
Natural extensions include (1)~adaptive routing that dynamically adjusts evidence allocation based on emerging disagreement, (2)~more expressive routing algorithms beyond BM25, such as dense retrieval or learned semantic partitioning, (3)~learned aggregation weights replacing the fixed shrinkage parameters, (4)~evaluation on open-ended and multi-choice forecasting tasks, and (5)~heterogeneous model panels that leverage model-level diversity in addition to evidence-level asymmetry.

\section{Limitations}

  \paragraph{Theoretical assumptions.}
  Proposition~\ref{prop:decorrelation} rests on three idealizing assumptions:  a linear error decomposition, public--private independence, and symmetric private noise. In practice, documents retrieved for the same question share topical correlations,
  partially violating the independence assumptions and making the decorrelation bound
  approximate. The linear decomposition further assumes that public and private error components
  combine with fixed weights $\rho$ and $1{-}\rho$, a functional form that LLMs,
  as non-linear functions of their inputs, need not satisfy.

\paragraph{Scope of the components.}
API cost scales linearly with agents $\times$ rounds, which limits our ability to conduct large-scale experiments with more agents or deliberation rounds that would be needed to characterize convergence behavior.
Additionally, as our primary contribution is the multi-agent paradigm of designed information asymmetry, both evidence routing and aggregation employ standard off-the-shelf methods.
More sophisticated alternatives may yield further gains.

\section{Ethics Considerations}
This work presents a forecasting framework and does not directly enable disinformation or harm beyond risks inherent to LLM-based systems generally. However, automated probability estimates generated at scale could in principle be misused to manipulate market sentiment or create false impressions of consensus. We recommend that InfoDelphi be used as a decision-support tool with human oversight in high-stakes deployment contexts.

The information structure itself carries a specific risk. Designed information asymmetry encourages agents to maintain distinct beliefs, but it can equally entrench a bad partition: if private evidence subsets contain low-quality, adversarial, or outdated documents, agents may become confidently polarized around incorrect conclusions rather than self-correcting through deliberation. Case F2 in Appendix~\ref{sec:appendix_cases} is a benign instance of this failure. In high-stakes domains such as finance or geopolitics, the same dynamic could reinforce speculative narratives or misleading impressions of certainty. Rationale sharing adds a second attack surface, since a persuasive but flawed argument propagates across the panel by design. Deployments should therefore pair the mechanism with evidence verification, source reliability estimation, and safeguards against adversarial persuasion.

\section{\textbf{\textsc{PolyGym}} Benchmark Details}
\label{sec:appendix_dataset}

\paragraph{Motivation.}
Agentic forecasting systems for prediction markets require two distinct capabilities:
(1)~\textit{retrieval}, identifying relevant information from the web, and
(2)~\textit{utilization}, reasoning over that information to produce calibrated probability estimates.
Existing benchmarks predominantly evaluate these capabilities jointly (Table~\ref{tab:benchmark_comparison}),
making it difficult to determine whether performance gains arise from better retrieval or better collective reasoning.
Also, live benchmarks~\citep{karger2024forecastbench,zeng2025futurex} introduce variability from changing web environments and unstable evidence availability.
Recent work shows that even retrieval-augmented forecasting systems continue to exhibit substantial performance degradation over time, suggesting that access to information alone does not solve the utilization
challenge~\citep{dailyoracle}.

\begin{table*}[!ht]
\centering
\small
\begin{tabular}{lccccc}
\toprule
\textbf{Setup} & \textbf{Static question set} & \textbf{Retrieval cutoff} & \textbf{Fixed evidence} & \textbf{Re-runnable} & \textbf{Agent study} \\
\midrule
ForecastBench~\citeyearpar{karger2024forecastbench} & \xmark & \cmark & \xmark & \xmark & \xmark \\
FutureX~\citeyearpar{zeng2025futurex} & \xmark & \cmark & \xmark & \xmark & \xmark \\
PolyBench~\citeyearpar{cheng2026polybenchbenchmarkingllmforecasting} & \xmark & \cmark & \xmark & \xmark & \xmark \\
Prediction Arena~\citeyearpar{zhang2026predictionarenabenchmarkingai} & \xmark & \cmark & \xmark & \xmark & \cmark \\
Zou et al.~\citeyearpar{zou2022forecasting} & \cmark & \xmark & \xmark & \xmark & \xmark \\
Halawi et al.~\citeyearpar{halawi2024approaching} & \cmark & \cmark & \xmark & \xmark & \xmark \\
\midrule
\textbf{\textsc{PolyGym} (Ours)} & \cmark & \cmark & \cmark & \cmark & \cmark \\
\bottomrule
\end{tabular}
\caption{Evaluation setups for LLM forecasting, along the axes that matter for studying information \emph{utilization}. \textbf{Static question set}: the question set is fixed rather than refreshed with newly opened markets. \textbf{Retrieval cutoff}: retrieval is constrained to precede the resolution date. \textbf{Fixed evidence}: the benchmark ships a frozen evidence corpus, so every method sees byte-identical inputs. \textbf{Re-runnable}: a later user can reproduce a published number exactly, which requires both a static question set and fixed evidence. \textbf{Agent study}: the release supports comparing multi-agent protocols on identical inputs. The axes are descriptive, not a quality ranking: the live benchmarks trade reproducibility for freshness, which is the right trade for leaderboard evaluation and the wrong one for isolating utilization. \textsc{PolyGym} is the setup that fixes the evidence, which is what lets us vary the information structure and nothing else.}
\label{tab:benchmark_comparison}
\end{table*}

Therefore, our work specifically studies a utilization question:
given the same global evidence pool, can a multi-agent system reason more effectively by distributing information asymmetrically and exchanging rationales?
To isolate this utilization dimension, we construct \textbf{\textsc{PolyGym}} as a controlled benchmark where all methods operate over the same pre-retrieved evidence corpus, while allowing different agents to observe different evidence subsets.
This design enables direct evaluation of whether information asymmetry itself improves deliberative forecasting, independent of retrieval quality.
Conceptually, \textsc{PolyGym} is related to controlled retrieval settings in QA research~\citep{chen2021finqa,yang2018hotpotqa} that separate retrieval quality from downstream reasoning.

\paragraph{Leakage threat model.}
Retrospective forecasting is unusually vulnerable to information leakage.
\citet{paleka2025pitfalls} show that outcomes can leak through question
selection, unreliable document dates, model knowledge, and retrieved human
forecasts.  More directly, \citet{ellahib2026temporal} find substantial post-cutoff content among the documents returned for a question even when search-engine date filters are enabled, and show that supplying such documents improves the apparent Brier score.  We therefore treat search-engine date filters only as a first-pass retrieval constraint rather than evidence of temporal safety. To mitigate these risks, we enforce leakage controls throughout the full data-collection pipeline, including \emph{market selection}, \emph{evidence collection}, and final \emph{filtering}.

\paragraph{Market selection.}
\textsc{PolyGym} is derived from 100,795 Polymarket markets
\citep{kaggle_polymarket_dataset}. We retain closed binary Yes/No markets whose recorded close is strictly after September 1, 2025 and whose final Yes price is at least 0.9 or at most 0.1, so that the market itself had converged on an unambiguous outcome. We exclude the Crypto category before sampling: in preliminary experiments these questions were the most sensitive to retrieved evidence, while continuously updated price, exchange, and market pages make their historical state especially difficult to verify. Crypto is detected by a word-boundary match over the question and the first 500 characters of its description against \texttt{bitcoin, btc, ethereum, eth, solana, crypto, token, airdrop, defi, nft, xrp, usdt, binance}. The same rule set assigns the remaining topical labels in Table~\ref{tab:dataset_composition}.  These filters leave 13,605 eligible non-Crypto markets, from which we uniformly sample 250 with seed 42. The resulting set contains
58 Yes and 192 No outcomes (base rate $0.232$) and has resolution dates from September 9 to December 2, 2025.

\begin{table}[t]
\centering
\small
\setlength{\tabcolsep}{4pt}
\caption{\textsc{PolyGym}-250 composition. Topic labels come from the same word-boundary rules used for the Crypto exclusion, so \emph{Other} is a residual bucket rather than a curated category. The set is Sports-heavy and NO-leaning throughout, which is why we prioritize Brier score and AUC over accuracy alone.}
\label{tab:dataset_composition}
\begin{tabular}{@{}lccc@{}}
\toprule
Topic & Questions & Yes & Yes rate \\
\midrule
Sports        & 106 & 31 & 0.29 \\
Other         &  79 & 13 & 0.16 \\
Business      &  28 &  6 & 0.21 \\
Politics      &  23 &  5 & 0.22 \\
Entertainment &   6 &  2 & 0.33 \\
Geopolitics   &   5 &  1 & 0.20 \\
Macro         &   2 &  0 & 0.00 \\
ScienceTech   &   1 &  0 & 0.00 \\
\midrule
Total         & 250 & 58 & 0.23 \\
\bottomrule
\end{tabular}
\end{table}

\paragraph{Evidence retrieval.}
For each question, we define the forecasting boundary as the earliest available configured end date, actual close time, or UMA end date, and exclude the entire resolution day. A separate LLM (\texttt{gpt-5-mini}, temperature $0.2$) generates five neutral, non-overlapping queries covering primary sources, historical context, relevant entities, schedules or leading indicators, and independent analysis. Queries that directly ask for the outcome, current status, or market resolution are prohibited.

Each query retrieves up to ten Tavily results\footnote{\url{https://www.tavily.com/}} under the specified cutoff. Results are deduplicated and filtered for temporal safety, and additional pre-cutoff queries are issued when necessary. We retain at most 20 documents per question. The resulting corpus contains an average of $19.50$ documents for \textsc{PolyGym} and $18.25$ for FutureX. Importantly, retrieval and leakage review are performed by a weaker model than the \texttt{gpt-5.4-mini} forecasting backbone, reducing the risk that corpus construction is tailored to the forecaster.

\paragraph{Fail-closed leakage filtering.}
Every candidate passes additional leakage checks beyond the API date filter. We first extract dates from the title, URL, and snippet, and any post-cutoff date, including either interpretation of an ambiguous numeric date, triggers quarantine. A zero-temperature LLM reviewer then jointly examines the question, cutoff, and retrieved metadata, rejecting post-cutoff updates, completed-event language, final values, current page states, and any direct or indirect outcome disclosure. Ambiguous cases and malformed reviews are dropped by default.

Background evidence containing limited post-cutoff information may be retained only after all unsafe content and the original URL are removed and an independent second review confirms temporal safety. These safeguards substantially reduce observable leakage, but cannot reconstruct a historical search engine's ranking, which may itself encode future relevance \citep{paleka2025pitfalls}. We therefore describe \textsc{PolyGym} as an aggressively audited, fixed retrospective corpus rather than a provably leakage-free historical snapshot.

\paragraph{FutureX questions and evidence.}
We obtain the historical Yes/No questions from the released FutureX
dataset. We normalize the released
answer fields to Yes/No labels and use each question's
\texttt{end\_time} as its resolution boundary, resulting in 231 evaluation
questions. We then construct a fixed evidence corpus using the same
retrieval and leakage-control pipeline as for \textsc{PolyGym}.

\paragraph{Licenses and intended use.}
The source Polymarket dataset~\citep{kaggle_polymarket_dataset} is licensed under
the Open Data Commons Database Contents License (DbCL) v1.0.
Evidence documents retrieved via the Tavily Search API are used
solely for non-commercial research purposes under its API terms
of service and are not independently redistributed. \textsc{PolyGym} is
intended for research use only. The accompanying code is released
under the MIT License.

\section{Derivation of Eq.~\eqref{eq:bvd}}
\label{sec:appendix_bvd}

Fix a question with outcome $y$ and write each agent's error as $p_j - y = (\bar{p} - y) + (p_j - \bar{p})$.
Squaring and averaging over the $J$ agents gives
\begin{align}
\frac{1}{J}\sum_{j=1}^{J}(p_j - y)^2
  &= (\bar{p} - y)^2 + \frac{1}{J}\sum_{j=1}^{J}(p_j - \bar{p})^2 \nonumber\\
  &\quad + \frac{2(\bar{p} - y)}{J}\sum_{j=1}^{J}(p_j - \bar{p}),
\end{align}
and the last term vanishes because $\sum_j (p_j - \bar{p}) = 0$.
Rearranging yields the ambiguity decomposition of \citet{krogh1995neural},
\begin{equation}
(\bar{p} - y)^2 \;=\; \frac{1}{J}\sum_{j=1}^{J}(p_j - y)^2 \;-\; \frac{1}{J}\sum_{j=1}^{J}(p_j - \bar{p})^2,
\label{eq:ambiguity}
\end{equation}
which holds for every realization of the forecasts.
We now take expectations over the agents' randomness.
For the first term, $\mathbb{E}(p_j - y)^2 = (\mathbb{E}\,p_j - y)^2 + \mathrm{Var}(p_j)$, and exchangeable agents share one mean, so $(\mathbb{E}\,p_j - y)^2 = \mathrm{Bias}^2$ for all $j$.
For the second term, exchangeability gives $\mathbb{E}(p_j - \bar{p}) = 0$, so $\mathbb{E}(p_j - \bar{p})^2 = \mathrm{Var}(p_j - \bar{p})$.
Substituting both into the expectation of Eq.~\eqref{eq:ambiguity} gives $\mathrm{MSE}(\bar{p}) = \mathrm{Bias}^2 + \bar{\sigma}^2 - \mathrm{Div}$, which is Eq.~\eqref{eq:bvd}.

\section{Proof of Proposition~\ref{prop:decorrelation}}
\label{sec:appendix_proof}

We provide the complete derivation.
Let $\epsilon_j = \rho \cdot \epsilon^{\mathrm{pub}} + (1 - \rho) \cdot
\epsilon_j^{\mathrm{priv}}$ for agent $j \in \{1, \dots, J\}$.

\paragraph{Covariance between agents.}
For $i \neq j$, expanding by bilinearity of covariance:
  \begin{align}
  \mathrm{Cov}(\epsilon_i, \epsilon_j)
    &= \rho^2 \mathrm{Var}(\epsilon^{\mathrm{pub}}) \nonumber\\
    &\quad + \rho(1{-}\rho)\mathrm{Cov}(\epsilon^{\mathrm{pub}}, \epsilon_j^{\mathrm{priv}}) \nonumber\\
    &\quad + \rho(1{-}\rho)\mathrm{Cov}(\epsilon_i^{\mathrm{priv}}, \epsilon^{\mathrm{pub}}) \nonumber\\
    &\quad + (1{-}\rho)^2\mathrm{Cov}(\epsilon_i^{\mathrm{priv}}, \epsilon_j^{\mathrm{priv}}).
  \end{align}

By the independence assumptions (private components are independent of each other and
of the public component): $\mathrm{Cov}(\epsilon_i^{\mathrm{priv}}, \epsilon_j^{\mathrm{priv}}) = 0$
and $\mathrm{Cov}(\epsilon^{\mathrm{pub}}, \epsilon_j^{\mathrm{priv}}) = 0$.
Therefore:
\begin{equation}
\mathrm{Cov}(\epsilon_i, \epsilon_j) = \rho^2 \mathrm{Var}(\epsilon^{\mathrm{pub}}).
\end{equation}

\paragraph{Individual variance.}
Since $\epsilon^{\mathrm{pub}}$ and $\epsilon_j^{\mathrm{priv}}$ are independent, the cross term vanishes:
\begin{equation}
\mathrm{Var}(\epsilon_j) = \rho^2 \mathrm{Var}(\epsilon^{\mathrm{pub}})
  + (1{-}\rho)^2 \mathrm{Var}(\epsilon_j^{\mathrm{priv}}).
\end{equation}

\paragraph{Correlation.}
  Let $\sigma_{\mathrm{pub}}^2 = \mathrm{Var}(\epsilon^{\mathrm{pub}})$ and assume agents
  are symmetric ($\mathrm{Var}(\epsilon_j^{\mathrm{priv}}) = \sigma_{\mathrm{priv}}^2$ for all $j$),
  which implies $\mathrm{Var}(\epsilon_i) = \mathrm{Var}(\epsilon_j)$ and simplifies
  $\mathrm{Corr}(\epsilon_i, \epsilon_j) = \mathrm{Cov}(\epsilon_i, \epsilon_j) /
  \sqrt{\mathrm{Var}(\epsilon_i)\mathrm{Var}(\epsilon_j)}$
  to $\mathrm{Cov}(\epsilon_i, \epsilon_j) / \mathrm{Var}(\epsilon_j)$:
\begin{equation}
\mathrm{Corr}(\epsilon_i, \epsilon_j)
  = \frac{\rho^2 \sigma_{\mathrm{pub}}^2}{\rho^2 \sigma_{\mathrm{pub}}^2 + (1{-}\rho)^2 \sigma_{\mathrm{priv}}^2}.
\end{equation}
 In the homogeneous baseline where $\rho = 1$ (all agents receive the same input):
  $\mathrm{Corr}(\epsilon_i, \epsilon_j) = 1$ (perfectly correlated errors).
  As $\rho$ decreases, correlation decreases monotonically, reaching 0 when $\rho = 0$.
  As an illustration, with $\rho = 0.5$ and equal-variance components,
  $\mathrm{Corr} = 0.25 / (0.25 + 0.25) = 0.5$, compared to $\mathrm{Corr} = 1$
  at $\rho = 1$. \qed

\paragraph{A document-level micro-foundation.}
The decomposition assumed in Proposition~\ref{prop:decorrelation} arises exactly in a simple model of evidence use.
Suppose each document $d$ contributes an error $\eta_d$ to any agent that reads it, with the $\eta_d$ independent across documents and of common variance $s^2$, and suppose an agent's error is the average over the $n = n_{\mathrm{pub}} + n_{\mathrm{priv}}$ documents in its view.
Writing $\pi = n_{\mathrm{pub}} / n$ for the public share of a view, the error of agent $j$ is
\begin{equation}
\epsilon_j \;=\; \pi \cdot \underbrace{\frac{1}{n_{\mathrm{pub}}}\sum_{d \in \Epub} \eta_d}_{\epsilon^{\mathrm{pub}}}
\;+\; (1 - \pi) \cdot \underbrace{\frac{1}{n_{\mathrm{priv}}}\sum_{d \in \Epriv{j}} \eta_d}_{\epsilon_j^{\mathrm{priv}}},
\end{equation}
which is the form assumed in the proposition with $\pi$ in the role of $\rho$.
Assumptions (i) and (ii) hold by construction because the private subsets are disjoint from one another and from the public pool, and (iii) holds because the round-robin deal gives the private subsets equal size up to rounding.
With $n_{\mathrm{pub}} = \rho N$ out of $N$ documents, $\pi = J\rho / (J\rho + 1 - \rho)$ is strictly increasing in $\rho$, so every qualitative statement about $\rho$ carries over to $\pi$.
Here $\sigma_{\mathrm{pub}}^2 = s^2 / n_{\mathrm{pub}}$ and $\sigma_{\mathrm{priv}}^2 = s^2 / n_{\mathrm{priv}}$, and Eq.~\eqref{eq:corr} reduces to $\mathrm{Corr}(\epsilon_i, \epsilon_j) = \pi$.
The model also gives the ensemble error $\bar{\epsilon} = \frac{1}{J}\sum_j \epsilon_j$ in closed form.
A homogeneous panel whose agents all read one view of $n$ documents has $\mathrm{Var}(\bar{\epsilon}) = s^2 / n$, while the partitioned panel has
\begin{equation}
\mathrm{Var}(\bar{\epsilon}) \;=\; \frac{s^2}{n}\Bigl(\pi + \frac{1 - \pi}{J}\Bigr),
\end{equation}
which is strictly smaller whenever $\pi < 1$ and $J \geq 2$.
At a fixed per-agent reading budget, partitioning therefore always lowers the ensemble error variance, by 17\% at our default $\rho = 0.5$ and $J = 3$, where $\pi = 0.75$.

\paragraph{Correlated private errors.}
Assumption (i) can be relaxed to $\mathrm{Cov}(\epsilon_i^{\mathrm{priv}}, \epsilon_j^{\mathrm{priv}}) = \rho_\epsilon\,\sigma_{\mathrm{priv}}^2$ for $i \neq j$ with $\rho_\epsilon \in [0, 1]$, which allows for private errors that co-move because the agents share one backbone.
The same expansion as above then gives
\begin{equation}
\mathrm{Corr}(\epsilon_i, \epsilon_j)
  = \frac{\rho^2 \sigma_{\mathrm{pub}}^2 + (1{-}\rho)^2 \rho_\epsilon\, \sigma_{\mathrm{priv}}^2}
         {\rho^2 \sigma_{\mathrm{pub}}^2 + (1{-}\rho)^2 \sigma_{\mathrm{priv}}^2},
\end{equation}
which is below one whenever $\rho < 1$ and $\rho_\epsilon < 1$, and the ensemble error has variance
\begin{equation}
\mathrm{Var}(\bar{\epsilon})
  = \rho^2 \sigma_{\mathrm{pub}}^2 + (1{-}\rho)^2 \sigma_{\mathrm{priv}}^2\,\frac{1 + (J{-}1)\rho_\epsilon}{J}.
\end{equation}
Homogeneous input corresponds to $\rho_\epsilon = 1$, where the private components coincide, so partitioning lowers the ensemble error variance by
\begin{equation}
\Delta \;=\; (1{-}\rho)^2 \sigma_{\mathrm{priv}}^2\,\frac{(J{-}1)(1 - \rho_\epsilon)}{J}.
\end{equation}
The reduction shrinks as $\rho_\epsilon$ grows but is strictly positive for every $\rho < 1$ and $\rho_\epsilon < 1$, and Proposition~\ref{prop:decorrelation} is the special case $\rho_\epsilon = 0$.

\section{Experimental Details}
\label{sec:appendix_details}

\subsection{Baseline Descriptions}

\paragraph{\ding{224} Single-Agent Methods.}
\begin{itemize}
    \item \textbf{Sequential Bayesian}~\citep{shi2023language}: Evidence items are processed one at a time by $K{=}5$ sequential agents, each updating the previous agent's probability estimate.
    \item \textbf{Chain-of-Thought}~\citep{wei2022chain}: One agent receives all evidence and produces a forecast with chain-of-thought reasoning.
    \item \textbf{Self-Consistency}~\citep{wangself}: The same CoT prompt is sampled 5 times at temperature 0.7. Then, predictions are averaged.
    \item \textbf{Halawi et al.}~\citep{halawi2024approaching}: A 7-step structured scratchpad prompt (rephrase, pro/con arguments, aggregate, calibration check) is run 3 times. Results are combined via trimmed mean that downweights the prediction furthest from the median by 50\%.
    \item \textbf{Superforecaster}~\citep{karger2024forecastbench}: A structured reasoning prompt that instructs the model to decompose the question, estimate base rates, update with evidence, and perform a calibration check.
\end{itemize}

\paragraph{\ding{224} Multi-Agent Methods (Homogeneous Input).}
\begin{itemize}
    \item \textbf{Standard Debate}~\citep{du2023improving}: 3 agents all receive the full evidence corpus, deliberate over 2 rounds sharing predictions and rationales, with final predictions aggregated by mean.
    \item \textbf{MoA}~\citep{wangmixture}: 3 proposer agents independently forecast with full evidence. The 4th aggregator agent receives all proposals with rationales and produces a final synthesized prediction.
    \item \textbf{Crowd Ensemble}~\citep{schoenegger2024wisdom}: 5 independent agents each receive all evidence and predict. Then, the median prediction is used as the final forecast.
    \item \textbf{AIA Forecaster}~\citep{alur2025aia}: 10 independent agents forecast with full evidence, and a supervisor agent identifies disagreements among their reasoning traces and produces a final prediction with a confidence label. If confidence is ``high'', the supervisor's prediction is used, and otherwise the system falls back to the mean of the 10 agents.
\end{itemize}

\subsection{Evidence Routing Implementation}

BM25 routing uses \texttt{rank\_bm25}'s Okapi implementation at its default parameters ($k_1 = 1.5$, $b = 0.75$). Documents are tokenized by lowercasing and matching \texttt{[a-zA-Z0-9]+} over the concatenation of title and content, and the query is the question text under the same tokenizer. Documents are ranked by descending BM25 score. The top $\lfloor \rho \lvert\mathcal{E}\rvert \rfloor$ form the public pool $\Epub$, and the remainder are dealt round-robin in rank order into the private subsets, so agent $j$ receives ranks $j, j{+}J, j{+}2J, \dots$ of the private tier. This keeps each agent's private documents as relevant as possible while leaving them disjoint. The random-routing ablation replaces the BM25 ranking with a seeded shuffle and is otherwise identical.

\subsection{Ablation Settings}

Table~\ref{tab:component_ablation} removes core components: at $R{=}1$ agents predict once with no communication (independent forecasting), and the homogeneous condition gives every agent the full pool with full sharing ($\rho{=}1$), which is the Standard Debate configuration run
through our pipeline.
Figure~\ref{fig:protocol_sensitivity} varies one dimension of the default configuration ($J{=}3$, $R{=}2$, $\rho{=}0.5$, BM25 routing, argument sharing, $T{=}0.7$, seed 42) at a time:

  \begin{itemize}[nosep]
    \item \textbf{Rounds}: $R{=}3$ adds a second exchange of arguments before the final prediction.
    \item \textbf{Public ratio}: $\rho \in \{0.3, 0.7, 1.0\}$.
    \item \textbf{Sharing}: the numeric-only variant replaces the argument summary with $(p_k, \ell_k)$ pairs and the round mean, with no
  rationale text.
    \item \textbf{Routing}: the random variant shuffles evidence with a fixed seed and assigns it round-robin, without BM25 ranking.
    \item \textbf{Temperature}: $T \in \{0, 0.3, 1.0\}$.
    \item \textbf{Seeds}: the default configuration is repeated with seeds 1 and 2 to estimate run-to-run variation.
  \end{itemize}

Throughout, raw scores pool final-round forecasts by the mean, and ``$+$\,calib.'' applies the calibrated shrinkage of Eq.~\eqref{eq:shrink} with the same parameters.

\subsection{Model Details}
\label{model_details}

Table~\ref{tab:model_details} summarizes the models used in our cross-model experiments. All models' knowledge cutoffs predate our temporal filter (resolution date $\geq$ 2025-09-01), ensuring no data leakage from training.

\begin{table}[h]
\centering
\small
\begin{tabular}{lcc}
\toprule
\textbf{Model} & \textbf{Knowledge Cutoff} & \textbf{Provider} \\
\midrule
\mname{logos/openai.pdf}{GPT-5.4-mini} & August 2025 & OpenAI \\
\mname{logos/deepseek.png}{DeepSeek-V3.2} & March 2025 & DeepSeek \\
\mname{logos/gemini.png}{Gemini-3.1-flash-lite} & January 2025 & Google \\
\bottomrule
\end{tabular}
\caption{Backbone knowledge cutoffs. All three precede the resolution dates of both benchmarks (\textsc{PolyGym}: 2025-09-09 to 2025-12-02, and FutureX: 2026-01-18 to 2026-06-30).}
\label{tab:model_details}
\end{table}

\subsection{Inference Cost}
\label{sec:appendix_cost}

Table~\ref{tab:cost} reports the per-question inference cost on \textsc{PolyGym}-250.
InfoDelphi makes the same six calls as Standard Debate but reads fewer input tokens, because each agent receives the public pool and its own private subset, about two thirds of the corpus, instead of every document.
The saving is about 20\% against Standard Debate and 23\% against the $\rho{=}1$ configuration of our own pipeline, which shares longer rationales between rounds.
Output length is essentially unchanged, so the saving comes entirely from the evidence partition.

\begin{table}[h]
\centering
\small
\setlength{\tabcolsep}{4.0pt}
\caption{Per-question inference cost on \textsc{PolyGym}-250, averaged over the 250 questions and taken from the token usage returned by the API. $^\dagger$The Standard Debate run did not log usage. Its input tokens are estimated by rebuilding its second-round messages from the saved agent outputs, its first-round prompts being identical to those of the $\rho{=}1$ configuration. Usage was not logged for the remaining baselines, which make between 1 and 11 calls per question (MoA 4, Crowd Ensemble 5, AIA Forecaster 11).}
\label{tab:cost}
\begin{tabular}{@{}lccc@{}}
\toprule
Method & Calls & Input tokens & Output tokens \\
\midrule
Chain-of-Thought & 1 & 3.3k & 0.2k \\
Standard Debate & 6 & 22.0k$^\dagger$ & -- \\
InfoDelphi, $\rho{=}1$ & 6 & 23.1k & 1.8k \\
\textbf{InfoDelphi} & 6 & \textbf{17.7k} & 1.8k \\
\bottomrule
\end{tabular}
\end{table}

\section{Shrinkage Coefficients Sensitivity}
\label{sec:appendix_shrinkage}
Table~\ref{tab:shrink_sensitivity} varies each shrinkage coefficient around the fixed default.
Every tested configuration improves over the untransformed mean on both datasets, although the dataset-specific best coefficients differ, so the calibration gain is not tied to a single coefficient choice.

This also qualifies the description of the miscalibration as one-sided: the overshoot is concentrated on the affirmative side, which is why $w_{\mathrm{hi}} < w_{\mathrm{lo}}$, but a mild low-side pull ($w_{\mathrm{lo}}{=}0.8$ rather than $1.0$) still helps on FutureX and costs little on \textsc{PolyGym}, so we keep it rather than tuning the transform to the dataset it was derived from.

We retain one globally fixed transform across all experiments and report the raw ensemble alongside it, keeping the contribution of calibration separable from that of information asymmetry.
Taken together, the ablation and the sweeps draw a consistent picture: the advantage comes from the principle of designed information asymmetry, not from a fortunate configuration.

\begin{table}[t]
\centering
\small
\setlength{\tabcolsep}{2.8pt}
\caption{Sensitivity to the fixed shrinkage coefficients
(Brier$\downarrow$). Each row changes one coefficient around
$(p_0,w_{\mathrm{lo}},w_{\mathrm{hi}})=(0.30,0.80,0.50)$.}
\label{tab:shrink_sensitivity}
\begin{tabular}{@{}lccccc@{}}
\toprule
Setting & $p_0$ & $w_{\mathrm{lo}}$ & $w_{\mathrm{hi}}$
        & \textsc{PolyGym} & FutureX \\
\midrule
No shrinkage
  & -- & -- & -- & 0.1647 & 0.2212 \\
Lower anchor
  & 0.20 & 0.80 & 0.50 & 0.1582 & 0.2049 \\
Fixed default
  & 0.30 & 0.80 & 0.50 & 0.1575 & 0.2006 \\
Higher anchor
  & 0.40 & 0.80 & 0.50 & 0.1590 & 0.1994 \\
Stronger low-side shrink
  & 0.30 & 0.60 & 0.50 & 0.1609 & 0.1996 \\
No low-side shrink
  & 0.30 & 1.00 & 0.50 & 0.1558 & 0.2030 \\
Stronger high-side shrink
  & 0.30 & 0.80 & 0.30 & 0.1569 & 0.2058 \\
Weaker high-side shrink
  & 0.30 & 0.80 & 0.70 & 0.1598 & 0.2025 \\
\bottomrule
\end{tabular}%
\end{table}

\paragraph{Shifted base rates.}
The anchor $p_0$ plays the role of a base rate, so it should move with the target population while the two weights should not.
To test this, we reweight the final InfoDelphi forecasts of each dataset to a target YES rate $r$, weighting the YES and NO questions separately, and apply the default weights with $p_0$ re-estimated as the label mean of 30 questions drawn at rate $r$ (Table~\ref{tab:baserate}).
With $p_0$ re-estimated, the calibration layer improves on the uncalibrated forecast at every base rate from 0.1 to 0.9 on both datasets, by 15\% on \textsc{PolyGym} and 13\% on FutureX at $r = 0.7$, and stays within 0.002 of the result obtained with the true base rate.
The weights $w_{\mathrm{lo}}$ and $w_{\mathrm{hi}}$ are left untouched throughout, which indicates that they capture an affirmative-side overconfidence of the forecaster rather than a property of either dataset.

\begin{table}[h]
\centering
\small
\setlength{\tabcolsep}{4.0pt}
\caption{Calibration under shifted base rates (Brier$\downarrow$). Each dataset is reweighted to a target YES rate $r$. ``$+$\,calib.'' keeps the default weights $w_{\mathrm{lo}}{=}0.8$ and $w_{\mathrm{hi}}{=}0.5$ and sets $p_0$ to the label mean of 30 questions.}
\label{tab:baserate}
\begin{tabular}{@{}ccccc@{}}
\toprule
 & \multicolumn{2}{c}{\textsc{PolyGym}} & \multicolumn{2}{c}{FutureX} \\
\cmidrule(lr){2-3}\cmidrule(lr){4-5}
$r$ & Brier & $+$\,calib. & Brier & $+$\,calib. \\
\midrule
0.1 & 0.1240 & \textbf{0.0913} & 0.1881 & \textbf{0.1042} \\
0.3 & 0.1856 & \textbf{0.1844} & 0.2140 & \textbf{0.1861} \\
0.5 & 0.2472 & \textbf{0.2363} & 0.2400 & \textbf{0.2253} \\
0.7 & 0.3087 & \textbf{0.2629} & 0.2660 & \textbf{0.2313} \\
0.9 & 0.3703 & \textbf{0.2650} & 0.2919 & \textbf{0.2128} \\
\bottomrule
\end{tabular}
\end{table}

\section{Shrinkage Control Experiment}
\label{sec:appendix_control}

Because the calibrated-shrink transform of Eq.~\eqref{eq:shrink} is a post-hoc map on the final probability, it can be granted to any method.
Table~\ref{tab:shrink_control} applies the identical transform (same parameters, no per-method tuning) to every baseline.
Three facts emerge.
  (1)~Shrinkage helps every method on both datasets, confirming that the affirmative-side
  overconfidence it corrects is a model-level bias rather than an artifact of our pipeline.
  (2)~The ranking is preserved: InfoDelphi remains the best system on both datasets.
  (3)~The comparison is honest about effect sizes: with all methods shrunk, several pairwise
  gaps, notably against Standard Debate, AIA Forecaster, Superforecaster, and MoA on \textsc{PolyGym}, are
  no longer individually significant at $n = 250$, although the gaps against 4 of 9 baselines on
  \textsc{PolyGym} and 6 of 9 on FutureX remain so.
  We therefore attribute InfoDelphi's advantage to the combination of information structure and
  calibrated aggregation, and report both raw and shrunk baselines so the reader can weigh each
  factor.

  \begin{table*}[t]
  \centering
  \small
\caption{Control experiment: the fixed calibration transform ($p_0{=}0.30$, $w_{\mathrm{lo}}{=}0.8$, $w_{\mathrm{hi}}{=}0.5$) applied to every baseline's output. It improves all methods and preserves the ranking, and InfoDelphi remains best on both datasets, although several pairwise gaps are no longer individually significant (n.s.). $\Delta$ is the paired Brier difference of InfoDelphi (full) vs.\ the calibrated baseline (negative favors ours), with 95\% percentile intervals over 5{,}000 paired bootstrap resamples over questions (seed 0).}
  \label{tab:shrink_control}
\begin{tabular}{lcccccc}
\toprule
   & \multicolumn{3}{c}{\textsc{PolyGym}-250} & \multicolumn{3}{c}{FutureX-231} \\
  \cmidrule(lr){2-4}\cmidrule(lr){5-7}
  Method & raw & $+$shrink & $\Delta$ [95\% CI] & raw & $+$shrink & $\Delta$ [95\% CI] \\
\midrule
  Chain-of-Thought & 0.2063 & 0.1738 & $-0.0163$ $[-0.0296, -0.0027]$ & 0.2611 & 0.2158 & $-0.0152$ $[-0.0327, +0.0029]$\,n.s. \\
  Self-Consistency & 0.1965 & 0.1706 & $-0.0132$ $[-0.0242, -0.0014]$ & 0.2713 & 0.2218 & $-0.0212$ $[-0.0373, -0.0042]$ \\
  Sequential Bayesian & 0.2523 & 0.2025 & $-0.0450$ $[-0.0586, -0.0306]$ & 0.3043 & 0.2411 & $-0.0405$ $[-0.0599, -0.0210]$ \\
  Superforecaster & 0.1937 & 0.1665 & $-0.0090$ $[-0.0217, +0.0040]$\,n.s. & 0.2596 & 0.2146 & $-0.0140$ $[-0.0308, +0.0029]$\,n.s. \\
  Halawi et al. & 0.1944 & 0.1693 & $-0.0118$ $[-0.0233, +0.0005]$\,n.s. & 0.2654 & 0.2224 & $-0.0218$ $[-0.0386, -0.0048]$ \\
  Standard Debate & 0.1815 & 0.1650 & $-0.0076$ $[-0.0154, +0.0003]$\,n.s. & 0.2433 & 0.2064 & $-0.0058$ $[-0.0221, +0.0111]$\,n.s. \\
  MoA & 0.1905 & 0.1670 & $-0.0095$ $[-0.0207, +0.0026]$\,n.s. & 0.2944 & 0.2326 & $-0.0321$ $[-0.0493, -0.0148]$ \\
  Crowd Ensemble & 0.1992 & 0.1701 & $-0.0127$ $[-0.0243, -0.0005]$ & 0.2888 & 0.2309 & $-0.0303$ $[-0.0478, -0.0125]$ \\
  AIA Forecaster & 0.1897 & 0.1651 & $-0.0076$ $[-0.0187, +0.0040]$\,n.s. & 0.2764 & 0.2223 & $-0.0217$ $[-0.0392, -0.0038]$ \\
\midrule
  InfoDelphi (full) & 0.1647 & \textbf{0.1575} & --- & 0.2212 & \textbf{0.2006} & --- \\
\bottomrule
\end{tabular}
\end{table*}

\paragraph{Comparison with Platt scaling.}
A standard alternative to Eq.~\eqref{eq:shrink} is Platt scaling, $p' = \sigma(a \cdot \mathrm{logit}(p) + b)$.
We fit $(a, b)$ to minimize Brier on one dataset and apply it to the other, and also report in-sample fits as an upper bound (Table~\ref{tab:platt}). Cross-fitted Platt improves the raw ensemble on both datasets but stays behind the asymmetric shrinkage, and even the in-sample fits do not reach it.
The reason is the shape of the miscalibration: the overshoot is concentrated on the affirmative side, so a monotone map in logit space must trade the calibrated low and mid band against the overshooting affirmative band, while the piecewise correction acts mainly on the affirmative side. Consistent with this, the fitted Platt parameters differ substantially across datasets ($a{=}0.562$, $b{=}-0.508$ on \textsc{PolyGym} vs.\ $a{=}0.379$, $b{=}-0.408$ on FutureX) whereas the
shrinkage setting transfers unchanged.
The decisive row is FutureX, where both transforms are genuinely out-of-sample: cross-fitted Platt reaches $0.2069$ and in-sample Platt $0.2033$, against $0.2006$ for a shrinkage rule whose coefficients never saw FutureX.\footnote{On \textsc{PolyGym} the comparison is less clean, since our shrinkage coefficients were themselves derived from \textsc{PolyGym}'s reliability profile. The like-for-like reading there is in-sample against in-sample, $0.1575$ vs.\ $0.1591$, and the ``Platt (cross)'' entry for \textsc{PolyGym} is reported for completeness and does not mirror our protocol.}

\begin{table}[t]
\centering
\small
\caption{Asymmetric shrinkage vs.\ Platt scaling on the final InfoDelphi forecasts (Brier$\downarrow$). Platt is fit to minimize Brier. ``Cross'' fits on the other dataset. ``In-sample'' fits and evaluates on the same dataset and is an upper bound unavailable in deployment. Our shrinkage coefficients are derived from \textsc{PolyGym} and applied unchanged to FutureX, so the FutureX column is the out-of-sample comparison for both transforms.}
\label{tab:platt}
\begin{tabular}{lcc}
\toprule
   & \textsc{PolyGym}-250 & FutureX-231 \\
\midrule
  Raw ensemble & 0.1647 & 0.2212 \\
  Platt (cross) & 0.1632 & 0.2069 \\
  Platt (in-sample) & 0.1591 & 0.2033 \\
  Asymmetric shrinkage & \textbf{0.1575} & \textbf{0.2006} \\
\bottomrule
\end{tabular}
\end{table}

\section{Significance Tests}
\label{sec:appendix_significance}

Table~\ref{tab:significance} reports the paired bootstrap intervals behind the daggers in Table~\ref{tab:main}: for each baseline, the mean question-level Brier difference between InfoDelphi (full) and the baseline, with the 95\% percentile interval over 5{,}000 resamples.

  \begin{table*}[t]
  \centering
  \small
  \caption{Paired bootstrap comparison of InfoDelphi (full) against each baseline (5{,}000 resamples over questions, seed 0). $\Delta$Brier $<0$ favors InfoDelphi, and all 95\% CIs exclude zero.}
  \label{tab:significance}
  \begin{tabular}{lcc}
  \toprule
  Baseline & \textsc{PolyGym}-250 $\Delta$Brier [95\% CI] & FutureX-231 $\Delta$Brier [95\% CI] \\
  \midrule
  Chain-of-Thought & $-0.0488$ $[-0.0756, -0.0217]$ & $-0.0606$ $[-0.0958, -0.0258]$ \\
  Self-Consistency & $-0.0390$ $[-0.0619, -0.0149]$ & $-0.0707$ $[-0.1032, -0.0381]$ \\
  Sequential Bayesian & $-0.0948$ $[-0.1209, -0.0681]$ & $-0.1037$ $[-0.1402, -0.0684]$ \\
  Superforecaster & $-0.0362$ $[-0.0620, -0.0097]$ & $-0.0591$ $[-0.0941, -0.0248]$ \\
  Halawi et al. & $-0.0370$ $[-0.0615, -0.0123]$ & $-0.0648$ $[-0.0966, -0.0338]$ \\
  Standard Debate & $-0.0240$ $[-0.0416, -0.0064]$ & $-0.0427$ $[-0.0761, -0.0114]$ \\
  MoA & $-0.0331$ $[-0.0573, -0.0086]$ & $-0.0938$ $[-0.1295, -0.0580]$ \\
  Crowd Ensemble & $-0.0418$ $[-0.0662, -0.0167]$ & $-0.0882$ $[-0.1224, -0.0534]$ \\
  AIA Forecaster & $-0.0323$ $[-0.0566, -0.0077]$ & $-0.0758$ $[-0.1103, -0.0409]$ \\
  \bottomrule
  \end{tabular}
  \end{table*}

\section{Case Study: Deliberation in Action}
\label{sec:appendix_case}

We walk through the question shown in Figure~\ref{fig:intro} to show how information asymmetry and argument sharing produce self-correction. Document ids are shortened for readability.

\paragraph{Question.} ``Will Sabrina Carpenter say `Nervous' during her SNL appearance on October 18?'' (\textbf{Ground truth: No.})

\paragraph{Round 1 (independent predictions).}
\begin{itemize}[nosep,leftmargin=15pt]
  \item \textbf{Agent 1} ($p = 0.68$): ``[doc\_009] is direct prior usage from Sabrina Carpenter saying `Nervous' in a recent interview
  transcript, showing the exact term is in her active vocabulary and likely to be used naturally.''
  \item \textbf{Agent 2} ($p = 0.72$): ``[doc\_008] and [doc\_012] confirm Sabrina Carpenter is scheduled to appear on SNL on October 18, 2025,
  so the event is on track. [doc\_009] shows Carpenter has recently used the exact word.''
  \item \textbf{Agent 3} ($p = 0.22$): ``For a specific word to be spoken on a live SNL appearance, the prior is well below 0.5 because it
  requires a particular utterance in a limited window. The strongest evidence is indirect.''
  \end{itemize}

Agents 1 and 2 treated a recent interview transcript containing the word as evidence it would recur on air, while Agent 3 argued that scheduling coverage confirms the appearance but not the utterance.

\paragraph{Round 2 (after argument sharing).}
\begin{itemize}[nosep,leftmargin=15pt]
  \item \textbf{Agent 1} ($0.68 \to 0.28$): ``Base rate is below 0.5 because this requires Sabrina Carpenter to utter the exact word `Nervous'
  during a limited SNL appearance. The evidence that she is scheduled to appear is solid; the utterance itself is not established.''
  \item \textbf{Agent 2} ($0.72 \to 0.24$): ``For a specific word to be spoken during a single SNL appearance requires a particular on-air
  utterance in a short window, and these markets usually resolve NO.''
  \item \textbf{Agent 3} ($0.22 \to 0.23$): ``[doc\_008] and [doc\_001] show the appearance is on track, which makes speech more likely than a
  brief cameo, but the specific utterance remains unlikely.''
  \end{itemize}

\paragraph{Outcome.} The round-2 mean is $0.25$, and after shrinkage the final forecast is $\hat{p} = 0.26$ (\textbf{Brier = 0.068}).

\paragraph{Contrast with homogeneous input.} Standard Debate on the same question converges on $p = 0.68$ (\textbf{Brier = 0.46}): with every agent reading the full pool, the salient interview transcript anchors all three views at once, and deliberation reinforces the shared reading.

\paragraph{Key insight.} The panel started two-to-one in the wrong direction, and the majority did not win. Partitioned evidence gave each agent an independent first reading, so one of three draws escaped the anchoring pull of the vivid transcript, whereas under the full pool, none did.
Because the between-round summary carries arguments and citations but no probabilities, Agent 3's position competed on its reasoning rather than as one vote against two, and there was no round mean pulling the panel toward the majority.
The argument that won on its merits is that prior usage in an interview does not establish an utterance on one specific broadcast.
In round 2 both majority agents restated this base-rate reasoning in their own rationales and dropped by more than 0.4, while Agent 3 held its estimate: the correction runs from the wrong side toward the right one, not toward a compromise midpoint.

 \section{Prompt Templates}
  \label{sec:appendix_prompts}

  We list the exact prompts used in InfoDelphi. All agents receive the same system prompt and structured user prompt, and agents differ only in
  their assigned evidence subset.

  \paragraph{System prompt (calibrated, used in all main results).}

  \begin{quote}
  \small\ttfamily
  You are a forecaster for binary (Yes/No) prediction markets. Use only the evidence provided. Prefer high-credibility sources when they
  conflict.\\[2pt]
  Follow this procedure:\\
  1. BASE RATE: before weighing the evidence, estimate what fraction of questions like this one resolve YES. "Will X happen by <date>"
  questions usually require a specific chain of events to complete in time, so they resolve NO more often than the surrounding discussion
  suggests.\\
  2. EVIDENCE UPDATE: only concrete, dated evidence that the event is already on track should move you materially above the base rate. Topical
  coverage, speculation, or enthusiasm is not evidence for YES.\\
  3. SHRINK: if the evidence is indirect, incomplete, or stale, move your estimate toward the base rate --- not toward 0.5.\\
  4. CALIBRATION CHECK: out of 100 questions where you would give this exact probability, how many should actually resolve YES? Adjust if your
  number feels like a hedge or a headline.\\[2pt]
  You must respond with valid JSON only, no other text before or after.
  \end{quote}

  \paragraph{User prompt (round 1).}

  \begin{quote}
  \small\ttfamily
  Question: \{question\}\\[4pt]
  Evidence:\\
  {[}1{]} doc\_id=doc\_001 | source=... | title=... | content: ...\\
  {[}2{]} doc\_id=doc\_002 | source=... | title=... | content: ...\\
  ...\\[4pt]
  Respond with exactly one JSON object, no other text. Schema: \{"p\_yes": <0-1>, "label": "YES"|"NO", "rationale": "...", "evidence\_used":
  {[}"doc\_id", ...{]}\}\\[2pt]
  In "rationale", cite each document you rely on by its doc\_id (e.g. {[}doc\_003{]}) together with the concrete fact it contributes, and list
  those doc\_ids in "evidence\_used". Arguments without a citation will be ignored by other analysts.
  \end{quote}

  \paragraph{User prompt (round 2).}
  The round-2 prompt appends a summary of the other agents' round-1 arguments after the evidence block. Consistent with Eq.~\eqref{eq:share},
  the summary contains rationales and cited document ids only: no probabilities, no labels, and no round mean.

  \begin{quote}
  \small\ttfamily
  Previous round summary (other agents):\\
  Colleagues who saw partially different evidence shared the arguments below.\\
  Weigh each argument by the concrete evidence it cites; ignore unsupported claims.\\
  \quad Colleague 1 (cites: doc\_009): A recent interview transcript shows Carpenter saying the exact word {[}doc\_009{]}... \\
  \quad Colleague 2 (cites: doc\_008, doc\_001): The appearance is confirmed for October 18th {[}doc\_008{]}, but no document establishes the
  utterance itself...\\
  Re-examine your own evidence in light of these arguments, then give your own independent probability. Do not converge for the sake of
  agreement.
  \end{quote}

  With evidence pooling, any private document cited in a round-1 rationale is added in full to every agent's evidence block in round 2, so the
  example documents above become directly inspectable rather than second-hand claims.

  \paragraph{Output format.}
  Each agent returns a JSON object:
  \begin{quote}
  \small\ttfamily
  \{"p\_yes": 0.85, "label": "YES", "rationale": "The strongest evidence from doc\_003 indicates...", "evidence\_used": {[}"doc\_001",
  "doc\_003"{]}\}
  \end{quote}

\section{Qualitative Analysis}
  \label{sec:appendix_cases}

    We present two representative successes and two failures from the full InfoDelphi run on \textsc{PolyGym}, illustrating the mechanisms through which information asymmetry and deliberation help or fail.

  \begin{successcase}{S1: A Minority Reading Overturns a Confident Majority}
  \textbf{Q:} \textit{Will Trump meet with Putin by November 30?}\\[2pt]
  \textbf{Ground truth:} NO \quad \textbf{InfoDelphi $\hat{p}$:} 0.11 \quad \textbf{Brier:} 0.012

  \smallskip
  \setlength{\tabcolsep}{4pt}
  \renewcommand{\arraystretch}{1.2}
  \begin{tabularx}{\linewidth}{@{}l c R@{}}
  \toprule
  \textbf{Agent} & $p^{(1)}\!\to\!p^{(2)}$ & \textbf{Round-1 evidence cited} \\
  \midrule
  \agA{A$_1$} & $0.97 \to 0.07$ & pre-summit coverage: ``will meet Putin next Friday in Alaska'' \\
  \agB{A$_2$} & $0.03 \to 0.04$ & the Aug.\ 15 summit already occurred, no new scheduling \\
  \agC{A$_3$} & $0.97 \to 0.08$ & reports that Trump and Putin met on Aug.\ 15, 2025 \\
  \bottomrule
  \end{tabularx}

  \smallskip
  \textit{Mechanism.} The market opened after the August 15 Alaska summit and asks about a further meeting.
  A$_1$ and A$_3$ read pre-summit and summit coverage as satisfying the question, while A$_2$'s private documents framed the summit as a completed
  past event with no follow-up scheduled.
  After argument sharing, both confident agents checked their claims against the cited dates, found that every concrete date referred to the
  completed August meeting, and dropped to 0.07 and 0.08.
  A single agent holding the correct reading of shared facts overturned a two-to-one majority through cited arguments alone.
  \end{successcase}

  \begin{successcase}{S2: A Misread Statistic Is Corrected}
  \textbf{Q:} \textit{Will Elon Musk post 500+ tweets from November 7 to November 14, 2025?}\\[2pt]
  \textbf{Ground truth:} NO \quad \textbf{InfoDelphi $\hat{p}$:} 0.24 \quad \textbf{Brier:} 0.058

  \smallskip
  \setlength{\tabcolsep}{4pt}
  \renewcommand{\arraystretch}{1.2}
  \begin{tabularx}{\linewidth}{@{}l c R@{}}
  \toprule
  \textbf{Agent} & $p^{(1)}\!\to\!p^{(2)}$ & \textbf{Round-1 evidence cited} \\
  \midrule
  \agA{A$_1$} & $0.12 \to 0.18$ & 500+ in seven days requires sustained extreme activity \\
  \agB{A$_2$} & $0.62 \to 0.18$ & study: ``more than 100 interactions a day over five months'' \\
  \agC{A$_3$} & $0.08 \to 0.31$ & general background only, no dated evidence for the window \\
  \bottomrule
  \end{tabularx}

  \smallskip
  \textit{Mechanism.} A$_2$ extrapolated a historical rate into the target window: at 100 per day, seven days clears 500.
  The other agents pointed out that the cited figure counts interactions, including retweets and replies, that the study predates the window,
  and that no document reports an active posting burst.
  After argument sharing, A$_2$ accepted the distinction and dropped from 0.62 to 0.18, a move of $0.44$ toward the correct side.
  The two agents that were already on the correct side drifted up (A$_1$ $0.12 \to 0.18$, A$_3$ $0.08 \to 0.31$), away from the outcome, so the mechanism is not uniformly beneficial here.
  The net effect is still a gain, because A$_2$'s correction is several times the size of the two upward drifts: the panel mean falls from $0.27$ in round 1 to $0.22$ in round 2, which shrinkage then maps to the reported $\hat{p} = 0.24$.
  What the agents converged on was the argument, not each other's numbers, which they never saw, and the drift in A$_1$ and A$_3$ is the cost of a channel that transmits reasoning rather than votes.
  \end{successcase}

  \begin{failurecase}{F1: The Corpus Lacks the Decisive Information}
  \textbf{Q:} \textit{Will Jaelan Phillips be traded before the 2025 NFL trade deadline?}\\[2pt]
  \textbf{Ground truth:} YES \quad \textbf{InfoDelphi $\hat{p}$:} 0.17 \quad \textbf{Brier:} 0.69

  \smallskip
  \setlength{\tabcolsep}{4pt}
  \renewcommand{\arraystretch}{1.2}
  \begin{tabularx}{\linewidth}{@{}l c R@{}}
  \toprule
  \textbf{Agent} & $p^{(1)}\!\to\!p^{(2)}$ & \textbf{Round-1 evidence cited} \\
  \midrule
  \agA{A$_1$} & $0.18 \to 0.12$ & Phillips still on the Dolphins' roster \\
  \agB{A$_2$} & $0.18 \to 0.14$ & roster and depth-chart discussion, no trade talk \\
  \agC{A$_3$} & $0.27 \to 0.14$ & deadline coverage without Phillips rumors \\
  \bottomrule
  \end{tabularx}

  \smallskip
  \textit{Mechanism.} No retrieved document reported that a trade was under discussion, so every agent reasoned from the base rate for
  player-specific deadline questions and answered NO, and deliberation reinforced the shared conclusion.
  The trade happened anyway.
  Rationale sharing cannot compensate for evidence that was never retrieved: the system's upper bound is the information content of its corpus.
  \end{failurecase}

  \begin{failurecase}{F2: A Uniformly Misleading Corpus Defeats the Partition}
  \textbf{Q:} \textit{Will Starbucks say ``Comparable Sales'' during earnings call?}\\[2pt]
  \textbf{Ground truth:} NO \quad \textbf{InfoDelphi $\hat{p}$:} 0.62 \quad \textbf{Brier:} 0.38

  \smallskip
  \setlength{\tabcolsep}{4pt}
  \renewcommand{\arraystretch}{1.2}
  \begin{tabularx}{\linewidth}{@{}l c R@{}}
  \toprule
  \textbf{Agent} & $p^{(1)}\!\to\!p^{(2)}$ & \textbf{Round-1 evidence cited} \\
  \midrule
  \agA{A$_1$} & $0.94 \to 0.94$ & results materials include a ``Comparable Store Sales'' table \\
  \agB{A$_2$} & $0.94 \to 0.94$ & prior earnings used the metric, call is scheduled \\
  \agC{A$_3$} & $0.96 \to 0.93$ & FY2024 release reports ``Change in Comparable Store Sales'' \\
  \bottomrule
  \end{tabularx}

  \smallskip
  \textit{Mechanism.} The market asks whether the phrase ``Comparable Sales'' will be spoken, whereas Starbucks materials use ``Comparable Store
  Sales,'' a different phrase under the resolution criteria.
  Every document that appears to support YES in fact shows the longer phrase, and all three private subsets sample the same materials, so the
  partition has nothing to diversify: each agent cites its own copy of the same trap and deliberation leaves the estimates unchanged.
  Shrinkage pulls the final forecast to 0.62 but cannot rescue a consensus built on a misread criterion.
  When the corpus is uniformly misleading, information asymmetry provides no protection, and the failure sits in reading the resolution criteria rather than in the information structure.
  \end{failurecase}

\end{document}